\documentclass{article}

\usepackage[preprint, nonatbib]{neurips_2019}

\usepackage[utf8]{inputenc} 
\usepackage[T1]{fontenc}    
\usepackage{hyperref}       
\usepackage{url}            
\usepackage{booktabs}       
\usepackage{amsfonts}       
\usepackage{nicefrac}       
\usepackage{microtype}      

\usepackage{amsmath,amssymb,amsfonts}
\usepackage{graphicx}
\usepackage[colorinlistoftodos]{todonotes}
\usepackage{amsthm}
\usepackage{cleveref}

\usepackage[noend]{algpseudocode}
\usepackage{algorithm}
\usepackage{algorithmicx}
\usepackage{subcaption}

\usepackage{adjustbox}

\def\R{{\mathbb{R}}}
\def\pr{{\rm Pr}}
\def\E{{\mathbb E}}
\def\X{{\mathcal X}}
\def\Y{{\mathcal Y}}

\def\B{{\mathcal B}}
\def\bias{{\rm bias}}
\def\supp{{\rm supp}}
\def\adv{{\rm adv}}

\newcommand{\cA}{\mathcal{A}}
\newcommand{\cB}{\mathcal{B}}
\newcommand{\cC}{\mathcal{C}}
\newcommand{\I}{\mathcal{I}}

\newcommand{\eps}{\epsilon}
\newcommand{\sign}{\mbox{sign}}
\newcommand{\algname}{\textsc{AKNN}}

\newtheorem{theorem}{Theorem}
\newtheorem*{theorem*}{Theorem}
\newtheorem{lemma}[theorem]{Lemma}

\newcommand{\comment}[3]{}  

\newcommand{\shay}[1]{\comment{purple}{Shay}{#1}}
\newcommand{\akshay}[1]{\comment{blue}{Akshay}{#1}}
\newcommand{\yoav}[1]{\comment{cyan}{Yoav}{#1}}

\title{An adaptive nearest neighbor rule for classification}

\author{
Akshay Balsubramani \\
\texttt{abalsubr@stanford.edu} \\
\And
Sanjoy Dasgupta \\
\texttt{dasgupta@eng.ucsd.edu} \\
\And
Yoav Freund \\
\texttt{yfreund@eng.ucsd.edu} \\
\And
Shay Moran\\
\texttt{shaym@princeton.edu} \\
}
\begin{document}

\maketitle

\begin{abstract}
We introduce a variant of the $k$-nearest neighbor classifier in which $k$ is chosen adaptively for each query, rather than supplied as a parameter. The choice of $k$ depends on properties of each neighborhood, and therefore may significantly vary between different points. (For example, the algorithm will use larger $k$ for predicting the labels of points in noisy regions.)  

We provide theory and experiments that demonstrate that the algorithm performs comparably to, and sometimes better than, $k$-NN with an optimal choice of $k$. In particular, we derive bounds on the convergence rates of our classifier that depend on a local quantity we call the ``advantage'' which is significantly weaker than the Lipschitz conditions used in previous convergence rate proofs. These generalization bounds hinge on a variant of the seminal Uniform Convergence Theorem due to Vapnik and Chervonenkis; this variant concerns conditional probabilities and may be of independent interest. 
\end{abstract}

\section{Introduction}

We introduce an adaptive nearest neighbor classification rule. Given a
training set with labels $\{\pm 1\}$, its prediction at a query point $x$ is based on the training points closest to $x$, rather like
the $k$-nearest neighbor rule. However, the value of $k$ that it uses
can vary from query to query. Specifically, if there are $n$ training
points, then for any query $x$, the smallest $k$ is sought for which
the $k$ points closest to $x$ have labels whose average is either
greater than $+\Delta(n,k)$, in which case the prediction is $+1$, or
less than $- \Delta(n,k)$, in which case the prediction is $-1$; and
if no such $k$ exists, then ``?'' (``don't know'') is returned.  
Here, $\Delta(n,k) \sim \sqrt{(\log n)/k}$ corresponds to a confidence interval for the average label in the region around the query.

We study this rule in the standard statistical framework in which all data are
i.i.d.\ draws from some unknown underlying distribution $P$ on $\X
\times \Y$, where $\X$ is the data space and $\Y$ is the label
space. We take $\X$ to be a separable metric space, with distance
function $d: \X \times \X \rightarrow \R$, and we take $\Y =
\{\pm 1\}$. We can decompose $P$ into the
marginal distribution $\mu$ on $\X$ and the conditional expectation of
the label at each point $x$: if $(X,Y)$ represents a random draw from
$P$, define $\eta(x) = \E(Y| X = x)$. In this terminology, the
Bayes-optimal classifier is the rule $g^*: \X \rightarrow \{\pm 1\}$
given by
\begin{equation}
g^*(x) = 
\left\{
\begin{array}{ll}
\sign(\eta(x)) & \mbox{if $\eta(x) \neq 0$} \\
\mbox{either $-1$ or $+1$} & \mbox{if $\eta(x) = 0$}
\end{array}
\right.
\label{eq:bayes-opt}
\end{equation}
and its error rate is the Bayes risk, $R^* = \frac{1}{2}\E_{X \sim \mu} \left[1-|\eta(X)| \right]$. A variety of nonparametric classification schemes are known to have error rates that converge asymptotically to $R^*$. These include $k$-nearest neighbor (henceforth, $k$-NN) rules~\cite{FH51} in which $k$ grows with the number of training points $n$ according to a suitable schedule $(k_n)$, under certain technical conditions on the metric measure space $(\X, d, \mu)$.

In this paper, we are interested in consistency as well as rates of
convergence. In particular, we find that the adaptive nearest neighbor
rule is also asymptotically consistent (under the same technical
conditions) while converging at a rate that is about as good as,
and sometimes significantly better than, that of $k$-NN
under any schedule $(k_n)$.

Intuitively, one of the advantages of $k$-NN over nonparametric
classifiers that use a fixed bandwidth or radius, such as Parzen
window or kernel density estimators, is that $k$-NN automatically
adapts to variation in the marginal distribution $\mu$: in regions
with large $\mu$, the $k$ nearest neighbors lie close to the query
point, while in regions with small $\mu$, the $k$ nearest neighbors
can be further afield. The adaptive NN rule that we propose goes
further: it also adapts to variation in $\eta$. In certain regions of
the input space, where $\eta$ is close to $0$, an accurate
prediction would need large $k$. In other regions, where $\eta$ is
near $-1$ or $1$, a small $k$ would suffice, and in fact, a larger $k$
might be detrimental because neighboring regions might be labeled
differently. See Figure~\ref{fig:rationale} for one such example. A
$k$-NN classifier is forced to pick a single value of $k$ that trades
off between these two contingencies. Our adaptive NN rule, however,
can pick the right $k$ in each neighborhood separately.

\begin{figure}
\begin{center}
\includegraphics[width=2.75in]{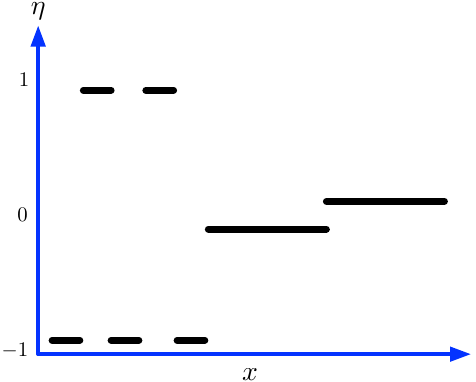}
\end{center}
\caption{For values of $x$ on the left half of the shown interval, the
  pointwise bias $\eta(x)$ is close to $-1$ or $1$, and thus a small value of $k$ will yield an accurate prediction. Larger $k$ will not do as well, because they may run into neighboring regions with different labels. For values of $x$ on the right half of the interval, $\eta(x)$ is close to $0$, and thus large $k$ is essential for accurate prediction.}
\label{fig:rationale}
\end{figure}

Our estimator allows us to give rates of convergence that are tighter and more 
transparent than those customarily obtained in nonparametric statistics. Specifically, for any point $x$ in the instance space $\X$, we define a notion of the {\it advantage at $x$}, denoted $\adv(x)$, which is rather like a local margin. We show that the prediction at $x$ is very likely to be correct once the number of training points exceeds~$\tilde{O}(1/\adv(x))$. Universal consistency follows by establishing that almost all points have positive advantage.

\subsection{Relation to other work in nonparametric estimation}

For linear separators and many other {\it parametric} families of classifiers, it is possible to give rates of convergence that hold without any assumptions on the input distribution $\mu$ or the conditional expectation function $\eta$. This is not true of nonparametric estimation: although any target function can in principle be captured, the number of samples needed to achieve a specific level of accuracy will inevitably depend upon aspects of this function such as how fast it changes~\cite[chapter 7]{DGL96}. As a result, nonparametric statistical theory has focused on (1) asymptotic consistency, ideally without assumptions, and (2) rates of convergence under a variety of smoothness assumptions.

Asymptotic consistency has been studied in great detail for the $k$-NN classifier, when $k$ is allowed to grow with the number of data points $n$. The risk of the classifier, denoted $R_n$, is its error rate on the underlying distribution $P$; this is a random variable that depends upon the set of training points seen. Cover and Hart~\cite{CH67} showed that in general metric spaces, under the assumption that every $x$ in the support of $\mu$ is either a continuity point of $\eta$ or has $\mu(\{x\}) > 0$, the expected risk $\E R_n$ converges to the Bayes-optimal risk $R^*$, as long as $k \rightarrow \infty$ and $k/n \rightarrow 0$. For points in finite-dimensional Euclidean space, a series of results starting with Stone~\cite{S77} established consistency without any assumptions on $\mu$ or $\eta$, and showed that $R_n \rightarrow R^*$ almost surely~\cite{DGKL94}. More recent work has extended these {\it universal consistency} results---that is, consistency without assumptions on $\eta$---to arbitrary metric measure spaces $(\X, d, \mu)$ that satisfy a certain differentiation condition~\cite{CG06,ChaudhuriDasgupta2014}.

Rates of convergence have been obtained for $k$-nearest neighbor classification under various smoothness conditions including Holder conditions on $\eta$~\cite{KP95,G81} and ``Tsybakov margin'' conditions~\cite{MT99,AT07,ChaudhuriDasgupta2014}. Such assumptions have become customary in nonparametric statistics, but they leave a lot to be desired. First, they are uncheckable: it is not possible to empirically determine the smoothness given samples. Second, they view the underlying distribution $P$ through the tiny window of two or three parameters, obscuring almost all the remaining structure of the distribution that also influences the rate of convergence. Finally, because nonparametric estimation is {\it local}, there is the intriguing possibility of getting different rates of convergence in different regions of the input space: a possibility that is immediately defeated by reducing the entire space to two smoothness constants.

The first two of these issues are partially addressed by recent work of \cite{ChaudhuriDasgupta2014}, who analyze the finite sample risk of $k$-NN classification without any assumptions on $P$. Their bounds involve terms that measure the probability mass of the input space in a carefully defined region around the decision boundary, and are shown to be ``instance-optimal'': that is, optimal for the specific distribution $P$, rather than minimax-optimal for some very large class containing $P$. However, the expressions for the risk are somewhat hard to parse, in large part because of the interaction between $n$ and $k$.

In the present paper, we obtain finite-sample rates of convergence that are ``instance-optimal'' not just for the specific distribution $P$ but also for the specific query point. This is achieved by defining a {\it margin}, or {\it advantage}, at every point in the input space, and giving bounds (Theorem~\ref{thm:pointwise-rate}) entirely in terms of this quantity. For parametric classification, it has become common to define a notion of margin that controls generalization. In the nonparametric setting, it makes sense that the margin would in fact be a function $\X \rightarrow \R$, and would yield different generalization error bounds in different regions of space. Our adaptive nearest neighbor classifier allows us to realize this vision in a fairly elementary manner.

\paragraph{Organization.} 
Proofs are relegated to the appendices.

We begin by formally defining the setup and notation in \Cref{sec:setup}.
Then, a formal description of the adaptive $k$-NN algorithm is given in~\Cref{sec:alg}.
In \Cref{sec:gen1,sec:gen2,sec:gen3}, we state and prove consistency and generalization
bounds for this classifier, and compare them with previous bounds in the $k$-NN literature.
These bounds exploit a general VC-based uniform convergence statement
which is presented and proved in a self-contained manner in \Cref{sec:ucecm}.

\section{Setup}\label{sec:setup}

Take the instance space to be a separable metric space $(\X, d)$ and the label space to be $\Y = \{\pm 1\}$. All data are assumed to be drawn i.i.d.\ from a fixed unknown distribution $P$ over $\X \times \Y$.

Let $\mu$ denote the marginal distribution on $\X$: if $(X,Y)$ is a 
random draw from $P$, then
$$ \mu(S) = \pr(X \in S)$$
for any measurable set $S \subseteq \X$. For any $x \in \X$, the conditional expectation, or {\em bias}, of $Y$ given $x$, is
$$ \eta(x) = \E(Y| X = x) \in [-1,1] .$$ 
Similarly, for any measurable set $S$ with $\mu(S) > 0$, the
conditional expectation of $Y$ given $X \in S$ is
$$ \eta(S) = \E(Y| X \in S) = \frac{1}{\mu(S)} \int_S \eta(x) \ d \mu(x) .$$

The risk of a classifier $g: \X \rightarrow \{-1,+1,?\}$ is the probability that it is incorrect on pairs $(X,Y) \sim P$,
\begin{equation}
R(g) = P(\{(x,y): g(x) \neq y\}).
\label{eq:risk}
\end{equation}
The Bayes-optimal classifier $g^*$, as given in (\ref{eq:bayes-opt}), depends only on $\eta$, but its risk $R^*$ depends on $\mu$. For a classifier $g_n$ based on $n$ training points from $P$, we will be interested in whether $R(g_n)$ converges to $R^*$, and the rate at which this convergence occurs.

The algorithm and analysis in this paper depend heavily on the probability masses and biases of balls in $\X$. For $x \in \X$ and $r \geq 0$, let $B(x,r)$ denote the closed ball of radius $r$ centered at $x$, 
$$ B(x,r) = \{ z \in \X : d(x,z) \leq r \} .$$
For $0 \leq p \leq 1$, let $r_p(x)$ be the smallest radius $r$ such that $B(x,r)$ has probability mass at least $p$, that is,
\begin{equation}
r_p(x) = \inf \{r \geq 0: \mu(B(x,r)) \geq p \}.
\label{eq:probability-radius}
\end{equation}
It follows that $\mu(B(x,r_p(x))) \geq p$.

The {\it support} of the marginal distribution $\mu$ plays an important role in convergence proofs and is formally defined as
$$ \supp(\mu) = \{x \in \X: \mbox{$\mu(B(x,r)) > 0$ for all $r > 0$}\} .$$
It is a well-known consequence of the separability of $\X$ that $\mu(\supp(\mu)) = 1$~\cite{CH67}.

\section{The adaptive $k$-nearest neighbor algorithm}\label{sec:alg}

The algorithm is given a labeled training set
$(x_1, y_1), \ldots, (x_n, y_n) \in \X \times \Y$.
Based on these points, it is able to compute empirical estimates of the probabilities and biases of different balls.

For any set $S \subseteq \X$, we define its empirical count and probability mass as
\begin{align}
\notag \#_n(S) &= |\{i: x_i \in S\}| \\
\mu_n(S) &= \frac{\#_n(S)}{n} .
\label{eq:empirical-mass}
\end{align}
If this is non-zero, we take the empirical bias to be
\begin{equation}
\eta_n(S) = \frac{\sum_{i: x_i \in S} y_i}{\#_n(S)} .
\label{eq:empirical-bias}
\end{equation}

\begin{figure}
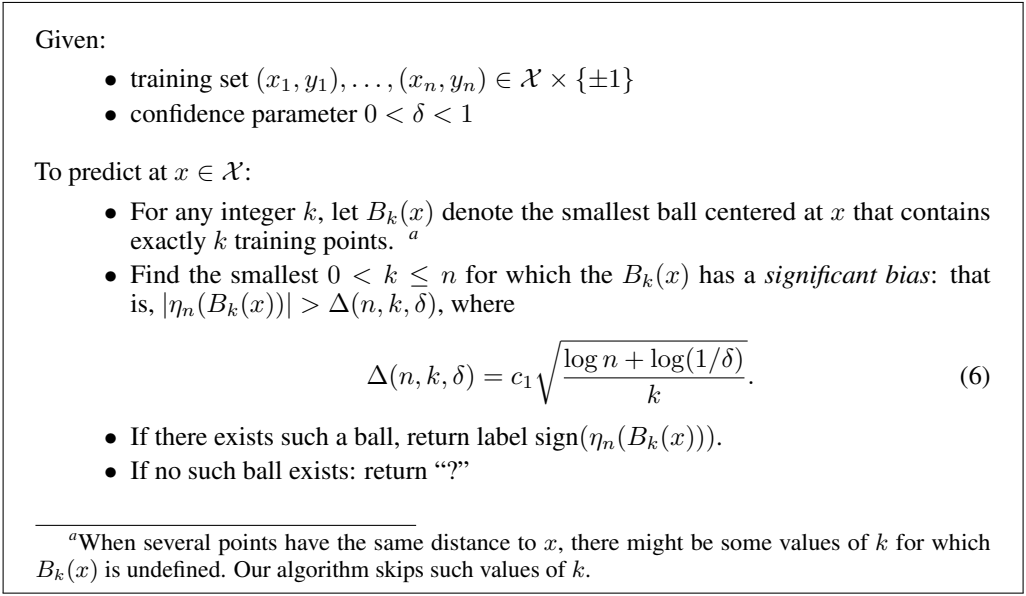

\fbox{\begin{minipage}{5.25in}
\vspace{.1in}
\begin{center}
\begin{minipage}{5in}
Given: 
\begin{itemize}
\item training set $(x_1, y_1), \ldots, (x_n, y_n) \in \X \times \{\pm 1\}$
\item confidence parameter $0 < \delta < 1$
\end{itemize}

\vspace{.1in}
To predict at $x \in \X$:
\begin{itemize}
\item For any integer $k$, let $B_k(x)$ denote the smallest ball centered at $x$ that contains exactly $k$ training points.
~\footnote{When several points have the same distance to $x$, there might be some values of $k$ for which $B_k(x)$ is undefined. Our algorithm skips such values of $k$.}
\item Find the smallest $0 < k \leq n$ for which the $B_k(x)$ has a {\em significant bias}: 
that is,~$\left|\eta_n(B_k(x))\right| > \Delta(n, k,\delta)$, where
\begin{equation}
\Delta(n,k,\delta) = c_1 \sqrt{\frac{\log n + \log (1/\delta)}{k}}.
\label{eq:delta-default}
\end{equation}
\item If there exists such a ball, return label $\sign(\eta_n(B_k(x)))$.
\item If no such ball exists: return ``?''
\end{itemize}
\vspace{.1in}
\end{minipage}
\end{center}
\end{minipage}}
\caption{The adaptive $k$-NN ($\algname$) classifier. The absolute constant $c_1$ is from Lemma~\ref{lemma:bias}.}
\label{fig:alg}
\end{figure}

The adaptive $k$-NN algorithm ($\algname$) is shown in Figure~\ref{fig:alg}. It makes a prediction at $x$ by growing a ball around $x$ until the ball has significant bias, and then choosing the corresponding label. In some cases, a ball of sufficient bias may never be obtained, in which event ``?'' is returned. In what follows, let $g_n: \X \rightarrow \{-1,+1,?\}$ denote the $\algname$ classifier.

Later, we will also discuss a variant of this algorithm in which a modified confidence interval,
\begin{equation}
\Delta(n,k,\delta) = c_1 \sqrt{\frac{d_0 \log n + \log (1/\delta)}{k}}
\label{eq:delta-modified}
\end{equation}
is used, where $d_0$ is the VC dimension of the family of balls in $(\X,d)$. 

\section{Pointwise advantage and rates of convergence}\label{sec:gen1}

We now provide finite-sample rates of convergence for the adaptive nearest neighbor rule. For simplicity, we give convergence rates that are specific to any query point $x$ and that depend on a suitable notion of the ``margin'' of distribution $P$ around $x$.

Pick any $p, \gamma > 0$. Recalling definition
(\ref{eq:probability-radius}), we say a point $x \in \X$ is $(p,
\gamma)$-salient
if the following holds for either $s=+1$ or $s=-1$:
\begin{itemize}
\item $s \eta(x) > 0$, and $s \eta(B(x,r)) > 0$ for all $r \in
  [0,r_p(x))$, and $s \eta(B(x,r_p(x))) \geq \gamma$.
\end{itemize}
In words, this means that $g^*(x)=s$ (recall that $g^*$ is the Bayes classifier),
that the biases of all balls of radius $\leq r_p(x)$ around $x$ have the same sign as $s$, and that
the bias of the ball of radius $r_p(x)$ has margin at least $\gamma$.
A point $x$ can satisfy this definition for a variety of pairs $(p,\gamma)$. The {\it advantage} of~$x$ is taken to be the largest value of $p\gamma^2$ over all such pairs:
\begin{equation}
\adv(x) = 
\left\{
\begin{array}{ll}
\sup \{ p \gamma^2: \mbox{$x$ is $(p,\gamma)$-salient}\} & \mbox{if $\eta(x) \neq 0$} \\
0 & \mbox{if $\eta(x) = 0$}
\end{array}
\right.
\label{eq:advantage}
\end{equation}
We will see (Lemma~\ref{lemma:positive-advantage}) that under a mild condition on the underlying metric measure space, almost all $x$ with $\eta(x) \neq 0$ have a positive advantage.

\subsection{Advantage-based finite-sample bounds}

The following theorem shows that for every point $x$,
if the sample size $n$ satisfies $n\gtrapprox 1/\adv(x)$,
then the label of $x$ is likely to be $g^*(x)$, where $g^*$ is the Bayes optimal classifier.
This provides a pointwise convergence of $g(x)$ to $g^*(x)$ with a rate which
is sensitive to the ``local geometry'' of~$x$.

\begin{theorem}[Pointwise convergence rate]
There is an absolute constant $C > 0$ for which the following holds.
Let $0 < \delta < 1$ denote the confidence parameter in the $\algname$ algorithm (\Cref{fig:alg}),
and suppose the algorithm is used to define a classifier $g_n$ based on $n$ training points chosen i.i.d.\ from $P$. 
Then, for every point $x\in\supp(\mu)$, if
\[n \geq \frac{C}{\adv(x)} \max \left( \log \frac{1}{\adv(x)}, \ \log \frac{1}{\delta} \right)\]
then with probability at least $1-\delta$ we have that $g_n(x)=g^*(x)$.
\label{thm:pointwise-rate}
\end{theorem}

If we further assume that the family of all balls in the space has finite VC dimension $d_o$ then 
we can strengthen \Cref{thm:pointwise-rate} so that the guarantee holds with high probability \underline{simultaneously} for all~$x\in \supp(\mu)$.
This is achieved by a modified version of the algorithm that uses confidence interval (\ref{eq:delta-modified}) instead of (\ref{eq:delta-default}).



\begin{theorem} [Uniform convergence rate]
Suppose that the set of balls in $(\X,d)$ has finite VC dimension $d_0$, and that the algorithm of Figure~\ref{fig:alg} is used with confidence interval (\ref{eq:delta-modified}) instead of (\ref{eq:delta-default}). 
Then, with probability at least $1-\delta$, the resulting classifier $g_n$ satisfies the following: 
for every point~$x\in\supp(\mu)$, if
\[n \geq \frac{C}{\adv(x)} \max \left( \log \frac{1}{\adv(x)}, \ \log \frac{1}{\delta} \right)\]
then $g_n(x)=g^*(x)$.
\label{thm:uniform-rate}
\end{theorem}

\shay{An observation which may be relevant: 
I think we can derive a rather clean statement by assuming that $\eta$ is {\it uniformly continuous}.
Recall that uniform continuity means that the quantitative relation between $\eps$ and $\delta$
in the definition of continuity is uniform over all $x\in \X$ 
(i.e.\ $(\forall \epsilon)(\exists\delta)(\forall x,y):d(x,y)<\delta \implies \lvert\eta(x)-\eta(y) \rvert<\epsilon$.)
Assuming this should give us uniform convergence of $g^*(x)\to g(x)$ for every $x$ outside the boundary of $\eta$:
namely, that for any $\alpha>0$ there exists $n(\alpha)$ 
such that if the algorithm is given $n(\alpha)$ samples then w.h.p it will predict like the Bayes optimal rule
for every $x$ such that $\lvert\eta(x)\rvert \geq \alpha$. }
\yoav{I think that continuity is a sufficient, but not a necessary condition, you can have
  significant advantage everywhere and continuity nowhere - just add
  $\epsilon>0$ to $\eta$ on some dense set of measure zero.}

A key step towards proving Theorems~\ref{thm:pointwise-rate} and \ref{thm:uniform-rate} is to identify the subset of $\X$ that is likely to be correctly classified for a given number of training points $n$. This follows the rough outline of \cite{ChaudhuriDasgupta2014}, which gave rates of convergence for $k$-nearest neighbor, but there are two notable differences. First, we will see that the likely-correct sets obtained in that earlier work (for $k$-NN) are subsets of those we obtain for the new adaptive nearest neighbor procedure. Second, the proof for our setting is considerably more streamlined; for instance, there is no need to devise tie-breaking strategies for deciding the identities of the $k$ nearest neighbors.
\shay{Perhaps add something about the uniform convergence statement which we use? 
I think this is another interesting feature.}

\subsection{A comparison with $k$-nearest neighbor}
\label{sec:knn-comparison}

For $a \geq 0$, let $\X_a$ denote all points with advantage greater than $a$:
\begin{equation}
\X_a = \{x \in \supp(\mu): \adv(x) > a \} .
\label{eq:advantage-set}
\end{equation}
In particular, $\X_0$ consists of all points with positive advantage. 

By Theorem~\ref{thm:pointwise-rate}, points in $\X_a$ are likely to be correctly classified when the number of training points is~$\widetilde{\Omega}(1/a)$, where the $\widetilde{\Omega}(\cdot)$ notation ignores logarithmic terms.
In contrast, the work of \cite{ChaudhuriDasgupta2014} showed that with $n$ training points, the $k$-NN classifier is likely to correctly classify the following set of points:
\begin{align*}
\X'_{n,k} &= \{x \in \supp(\mu): \eta(x) > 0, \eta(B(x,r)) \geq k^{-1/2} \mbox{\ for all $0 \leq r \leq r_{k/n}(x)$}\} \\
&\ \cup \{x \in \supp(\mu): \eta(x) < 0, \eta(B(x,r)) \leq -k^{-1/2} \mbox{\ for all $0 \leq r \leq r_{k/n}(x)$}\} .
\end{align*}
Such points are $(k/n, k^{-1/2})$-salient and thus have advantage at least $1/n$. In fact,
$$ \bigcup_{1 \leq k \leq n} \X'_{n,k} \subseteq \X_{1/n} .$$
In this sense, the adaptive nearest neighbor procedure is able to perform roughly as well as all choices of $k$ simultaneously (logarithmic factors prevent this from being a precise statement).

\yoav{In fact, in my experiments, AKNN performs {\em almost} as well
  as the best choice of $k$. Can we account for this difference? We
  need to make a statement that does not contradict the expriments...}

\section{Universal consistency}\label{sec:gen2}
\label{sec:universal-consistency}

In this section we study the convergence of $R(g_n)$ to the Bayes risk $R^*$ as the number of points $n$ grows. An estimator is described as universally consistent in a metric measure space $(\X, d, \mu)$ if it has this desired limiting behavior for all conditional expectation functions $\eta$.

Earlier work~\cite{ChaudhuriDasgupta2014} has established the universal consistency of $k$-nearest neighbor (for $k/n \rightarrow 0$ and $k/(\log n) \rightarrow \infty$) in any metric measure space that satisfies the Lebesgue differentiation condition: that is, for any bounded measurable $f: \X \rightarrow \R$ and for almost all ($\mu$-a.e.) $x \in \X$,
\begin{equation}
\lim_{r \downarrow 0} \frac{1}{\mu(B(x,r))} \int_{B(x,r)} f \ d\mu = f(x) .
\label{eq:lebesgue-condition}
\end{equation}
This is known to hold, for instance, in any finite-dimensional normed space or any doubling metric space~\cite[Chapter 1]{H01}.

We will now see that this same condition implies the universal consistency of the adaptive nearest neighbor rule. To begin with, it implies that almost every point has a positive advantage.
\begin{lemma}
Suppose metric measure space $(\X, d, \mu)$ satisfies condition (\ref{eq:lebesgue-condition}). Then, for any conditional expectation $\eta$, the set of points
$$ \{x \in \X: \eta(x) \neq 0, \, \adv(x) = 0\}$$
has zero $\mu$-measure.
\label{lemma:positive-advantage}
\end{lemma}
\begin{proof}
Let $\X' \subseteq \X$ consist of all points $x \in \supp(\mu)$ for which condition (\ref{eq:lebesgue-condition}) holds true with $f=\eta$, that is, $\lim_{r \downarrow 0} \eta(B(x,r)) = \eta(x)$. 
Since $\mu(\supp(\mu))=1$, it follows that $\mu(\X') = 1$. 

Pick any $x \in \X'$ with $\eta(x) \neq 0$; without loss of generality, $\eta(x) > 0$. By (\ref{eq:lebesgue-condition}), there exists $r_o > 0$ such that
$$\eta(B(x,r)) \geq \eta(x)/2 \mbox{\  for all \ } 0 \leq r \leq r_o.$$
Thus $x$ is $(p,\gamma)$-salient for $p = \mu(B(x,r_o)) > 0$ and $\gamma = \eta(x)/2$, and has positive advantage.
\end{proof}

Universal consistency follows as a consequence; the proof details are deferred to ~\Cref{sec:proof-outline}.
\begin{theorem}[Universal consistency]
Suppose the metric measure space $(\X, d, \mu)$ satisfies condition~(\ref{eq:lebesgue-condition}). Let $(\delta_n)$ be a sequence in $[0,1]$ with (1) $\sum_n \delta_n < \infty$ and (2) $\lim_{n \rightarrow \infty} (\log (1/\delta_n))/n = 0$. Let the classifier $g_{n, \delta_n}: \X \rightarrow \{-1,+1,?\}$ be the result of applying the $\algname$ procedure (Figure~\ref{fig:alg}) with $n$ points chosen i.i.d.\ from $P$ and with confidence parameter $\delta_n$. Letting $R_n = R(g_{n,\delta_n})$ denote the risk of $g_{n,\delta_n}$, we have $R_n \rightarrow R^*$ almost surely.
\label{thm:universal-consistency}
\end{theorem}

\shay{I want to see if we can make the point that the constraints on $\delta_n$ are rather intuitive:
regarding item (2): note that if $\delta_n << 2^{-n}$ is tiny then we essentially grow a ball around $x$
until one of the labels has an overwhelming majority of, say $n$. This makes no sense\ldots}

\begin{figure}[t]
\begin{center}
\includegraphics[width=0.9\linewidth, height=1.8in]{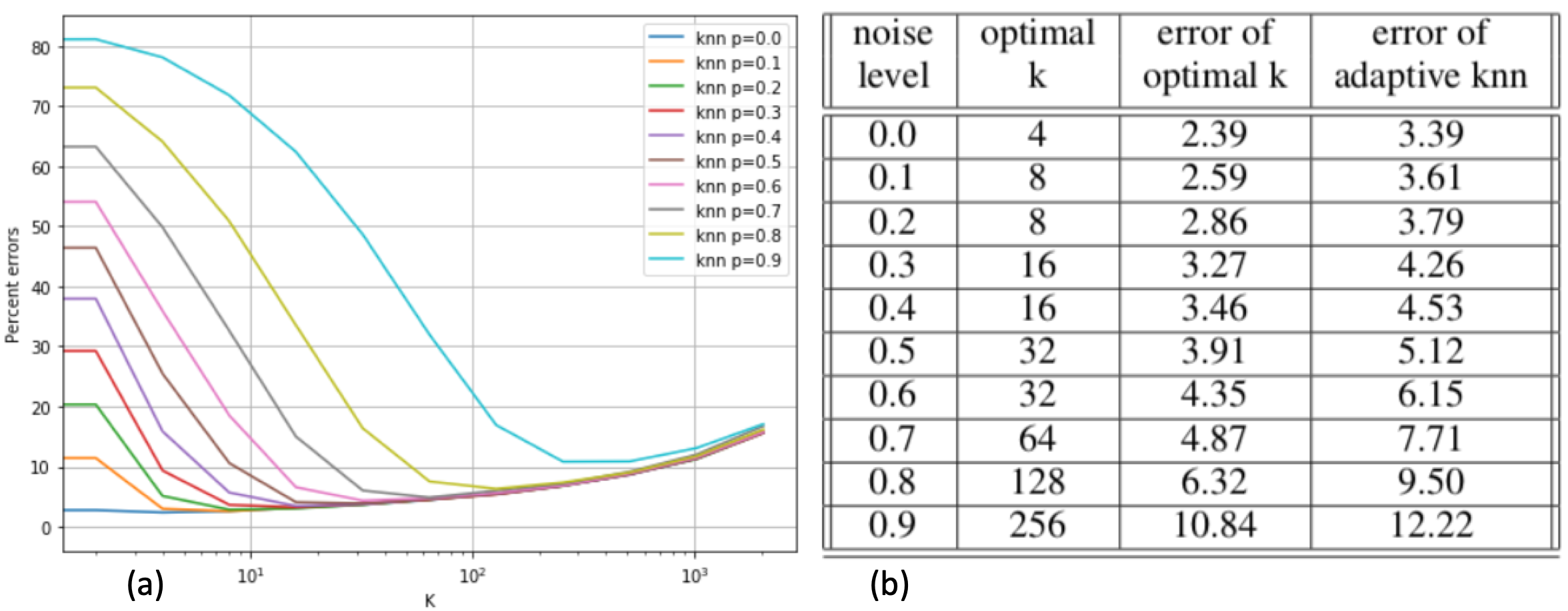}
\end{center}
\caption{
  Effect of label noise on $k$-NN and $\algname$.
  Performance on MNIST for different levels of random label noise $p$ and for different values of $k$. 
  Each line in the figure on the left {\bf (a)} represents the performance of $k$-NN as a function of $k$ for a given level of noise. 
  The optimal choice of $k$ increases with the noise level, and that the performance degrades severely for too-small $k$. 
  The table {\bf (b)} shows that $\algname$, with a fixed value of $A$, performs almost as well as $k$-NN with the optimal choice of $k$.
  }
\label{fig:mnist}
\end{figure}

\begin{figure}[th]
    \begin{subfigure}{0.5\textwidth}
    \centering
        \includegraphics[width=\linewidth]{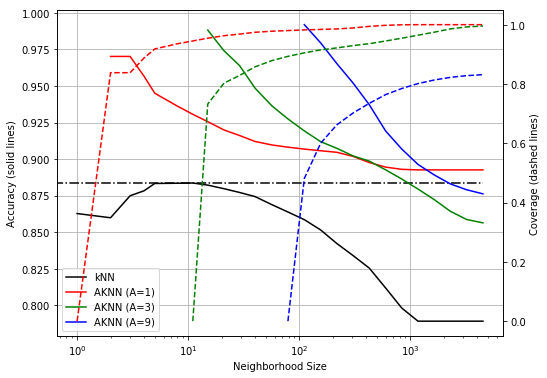}
    \end{subfigure}
    \hspace{0.2cm}
    \begin{minipage}{0.45\textwidth}
    At left: performance of $\algname$ on notMNIST for different settings of the confidence parameter ($A=1,3,9$), as a function of the neighborhood size.
    For each confidence level we show two graphs: an accuracy graph (solid line) and a coverage line (dashed line). 
    For each value of $k$ we plot the accuracy and the coverage of $\algname$ which is restricted to using a neighborhood size of at most $k$. 
    Increasing $A$ generally causes an increase in the accuracy and a decrease in coverage. 
    Larger values of $A$ cause $\algname$ to have coverage zero for values of $k$ that are too small. 
    For comparison, we plot the performance of $k$-NN as a function of $k$. 
    The highest accuracy ($\approx$ 0.88) is achieved for $k=10$ (dotted horizontal line), and is surpassed by $\algname$ with high coverage (100\% for $A=1$). 
    \end{minipage} 
    \caption{Performance of $\algname$ on notMNIST. See also Figure \ref{fig:varyingadak}.}
    \label{fig:aknnvsknn}
\end{figure}

\section{Experiments}

\newcommand{\labels}{\mathcal{Y}}

We performed a few experiments using real-world datasets from computer vision and genomics (see Section \ref{sec:experimentappendix}). 
These were conducted with some practical alterations to the algorithm of Fig.~\ref{fig:alg}. 

{\bf Multiclass extension:} Suppose the set of possible labels is  $\labels$. 
We replace the binary rule ``find the smallest $k$ such that $\left|\eta_n(B_k(x))\right| > \Delta(n, k,\delta)$" with the rule: ``find the smallest $k$ such that $\eta^y_n(B_k(x)) - \frac{1}{|\labels|} > \Delta(n, k,\delta)$ 
for some $y \in \labels$,  where 
$\eta^y_n(S) \doteq \frac{\#_n\{x_i \in S \mbox{ and } y_i = y\}}{\#_n(S)}$."
\newline
{\bf Parametrization:} 
We replace Equation~(\ref{eq:delta-default}) with $\Delta = \frac{A}{\sqrt{k}}$, where $A$ is a confidence parameter corresponding to the theory's $\delta$ (given $n$). \newline
{\bf Resolving multilabel predictions:} 
Our algorithm can output answers that are not a single label. 
The output can be ``?'', which indicates that no label has sufficient evidence. 
It can also be a subset of $\labels$ that contains more than one element, indicating that more than one label has significant evidence. 
In some situations, using subsets of the labels is more informative. 
However, when we want to compare head-to-head with $k$-NN, we need to output a single label. 
We use a heuristic to predict with a single label $y \in \labels$ on any $x$: the label for which $\max_k \eta^y_n(B_k(x))/\sqrt{k}$ is largest.






\begin{figure}[th]
    \centering
    \begin{subfigure}{0.5\textwidth}
    \centering
        \includegraphics[width=\linewidth]{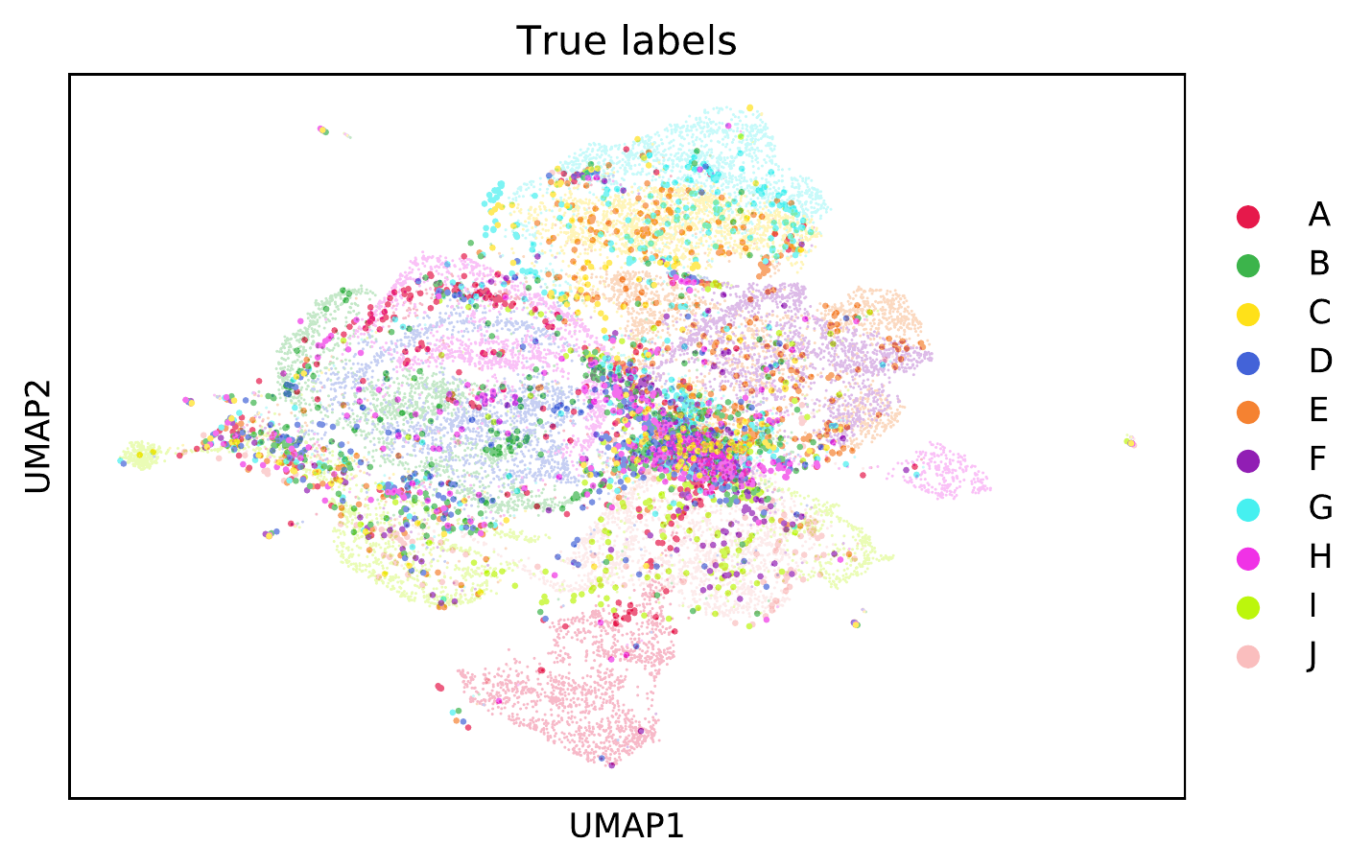}
    \end{subfigure}
    \begin{subfigure}{0.45\textwidth}
    \centering
        \includegraphics[width=\linewidth]{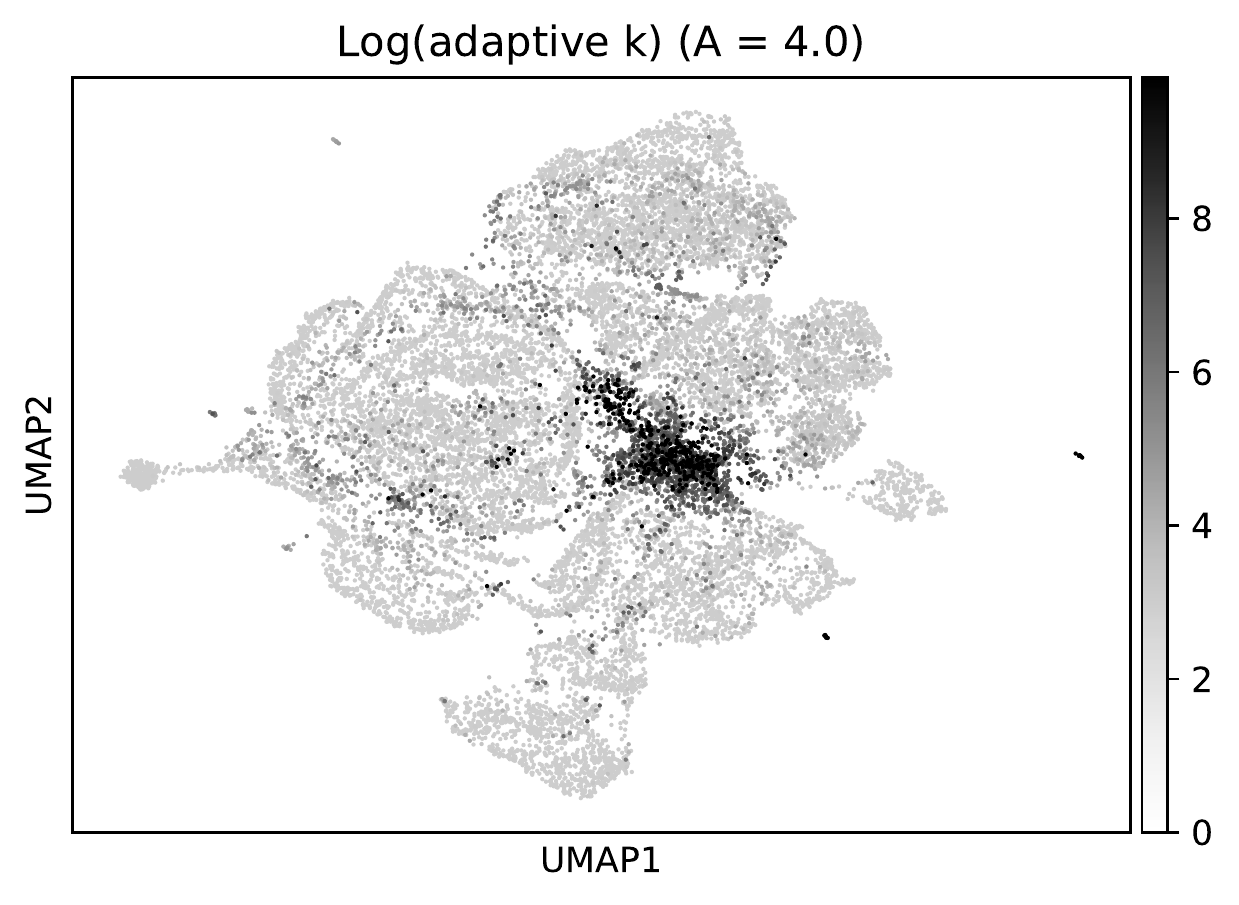}
    \end{subfigure}
    \caption{ A visualization of the performance of $\algname$ on notMNIST.
      {\bf (a)} The correct labels, with prediction errors of $\algname \;(A = 4)$ highlighted. 
      {\bf (b)} The value of $k$ chosen by the algorithm when predicting each datapoint. Zooming in reveals more details. An interactive explorer for our experiments is available at \texttt{http://35.239.251.24/aknn/}.}
  \label{fig:varyingadak}
\end{figure}

We briefly discuss our main conclusions from the experiments, with further details deferred to \Cref{sec:experimentappendix}.

{\bf $\algname$ is comparable to the best $k$-NN rule.}
In Section~\ref{sec:knn-comparison} we prove that $\algname$ compares favorably to $k$-NN with any fixed $k$. 
We demonstrate this in practice in different situations. With simulated independent label noise on the MNIST dataset (Fig.~\ref{fig:mnist}), a small value of $k$ is optimal for noiseless data, but performs very poorly when the noise level is high. 
On the other hand, $\algname$ adapts to the local noise level automatically, as demonstrated without adding noise on the more challenging notMNIST and single-cell genomics data (Fig. \ref{fig:aknnvsknn}, \ref{fig:varyingadak}, \ref{fig:aknnvsknn_muris}). 

{\bf Varying the confidence parameter $A$ controls abstaining.}
The parameter $A$ controls how conservative the algorithm is in deciding to abstain, instead of incurring error by predicting. 
$A \to 0$ represents the most aggressive setting, in which the algorithm never abstains, essentially predicting according to a $1$-NN rule. 
Higher settings of $A$ cause the algorithm to abstain on some of these predicted points, for which there is no sufficiently small neighborhood with a sufficiently significant label bias (Fig. \ref{fig:varyingparam}).

{\bf Adaptively chosen neighborhood sizes reflect local confidence.}
The number of neighbors chosen by $\algname$ is a local quantity that gives a practical pointwise measure of the confidence associated with label predictions. 
Small neighborhoods are chosen when one label is measured as significant nearly as soon as statistically possible; by definition of the $\algname$ stopping rule, this is not true where large neighborhoods are necessary. 
In our experiments, performance on points with significantly higher neighborhood sizes dropped monotonically, with the majority of the dataset having performance significantly exceeding the best $k$-NN rule over a range of settings of $A$ (Fig. \ref{fig:aknnvsknn}, \ref{fig:aknnvsknn_muris}; \Cref{sec:experimentappendix}).

\bibliography{ref}

\newcommand{\etalchar}[1]{$^{#1}$}
\begin{thebibliography}{DGKL94}

\bibitem[AT07]{AT07}
J.-Y. Audibert and A.B. Tsybakov.
\newblock Fast learning rates for plug-in classifiers.
\newblock {\em Annals of Statistics}, 35(2):608--633, 2007.

\bibitem[BBL05]{BBL05}
St{\'e}phane Boucheron, Olivier Bousquet, and G{\'a}bor Lugosi.
\newblock Theory of classification: A survey of some recent advances.
\newblock {\em ESAIM: probability and statistics}, 9:323--375, 2005.

\bibitem[C{\etalchar{+}}18]{tabulamuris18}
Tabula~Muris Consortium et~al.
\newblock Single-cell transcriptomics of 20 mouse organs creates a tabula
  muris.
\newblock {\em Nature}, 562(7727):367, 2018.

\bibitem[CD10]{ChaudhuriDasgupta2010}
K.~Chaudhuri and S.~Dasgupta.
\newblock Rates of convergence for the cluster tree.
\newblock In {\em Advances in Neural Information Processing Systems}, pages
  343--351, 2010.

\bibitem[CD14]{ChaudhuriDasgupta2014}
K.~Chaudhuri and S.~Dasgupta.
\newblock Rates of convergence for nearest neighbor classification.
\newblock In {\em Advances in Neural Information Processing Systems}, pages
  3437--3445. 2014.

\bibitem[CG06]{CG06}
F.~Cerou and A.~Guyader.
\newblock Nearest neighbor classification in infinite dimension.
\newblock {\em ESAIM: Probability and Statistics}, 10:340--355, 2006.

\bibitem[CH67]{CH67}
T.~Cover and P.E. Hart.
\newblock Nearest neighbor pattern classification.
\newblock {\em IEEE Transactions on Information Theory}, 13:21--27, 1967.

\bibitem[DCL11]{DCL11}
Wei Dong, Moses Charikar, and Kai Li.
\newblock Efficient k-nearest neighbor graph construction for generic
  similarity measures.
\newblock In {\em Proceedings of the 20th international conference on World
  wide web}, pages 577--586. ACM, 2011.

\bibitem[DGKL94]{DGKL94}
L.~Devroye, L.~Gy{\"o}rfi, A.~Krzyzak, and G.~Lugosi.
\newblock On the strong universal consistency of nearest neighbor regression
  function estimates.
\newblock {\em Annals of Statistics}, 22:1371--1385, 1994.

\bibitem[DGL96]{DGL96}
L.~Devroye, L.~Gy{\"o}rfi, and G.~Lugosi.
\newblock {\em A Probabilistic Theory of Pattern Recognition}.
\newblock Springer, 1996.

\bibitem[Dud79]{dudley79}
R.M. Dudley.
\newblock Balls in $\mathbb{R}^k$ do not cut all subsets of $k+ 2$ points.
\newblock {\em Advances in Mathematics}, 31(3):306--308, 1979.

\bibitem[FH51]{FH51}
E.~Fix and J.~Hodges.
\newblock Discriminatory analysis, nonparametric discrimination.
\newblock {\em USAF School of Aviation Medicine, Randolph Field, Texas, Project
  21-49-004, Report 4, Contract AD41(128)-31}, 1951.

\bibitem[Gy{\"o}81]{G81}
L.~Gy{\"o}rfi.
\newblock The rate of convergence of $k_n$-nn regression estimates and
  classification rules.
\newblock {\em IEEE Transactions on Information Theory}, 27(3):362--364, 1981.

\bibitem[Hei01]{H01}
J.~Heinonen.
\newblock {\em Lectures on Analysis on Metric Spaces}.
\newblock Springer, 2001.

\bibitem[KP95]{KP95}
S.~Kulkarni and S.~Posner.
\newblock Rates of convergence of nearest neighbor estimation under arbitrary
  sampling.
\newblock {\em IEEE Transactions on Information Theory}, 41(4):1028--1039,
  1995.

\bibitem[MNI]{MNIST}


\bibitem[Mou18]{MouseAtlasData}
Mouse cell atlas dataset.
\newblock
  \url{ftp://ngs.sanger.ac.uk/production/teichmann/BBKNN/MouseAtlas.zip}, 2018.
\newblock Accessed: 2019-05-02.

\bibitem[MT99]{MT99}
E.~Mammen and A.B. Tsybakov.
\newblock Smooth discrimination analysis.
\newblock {\em The Annals of Statistics}, 27(6):1808--1829, 1999.

\bibitem[not11]{notMNIST}
notmnist dataset.
\newblock \url{http://yaroslavvb.com/upload/notMNIST/}, 2011.
\newblock Accessed: 2019-05-02.

\bibitem[RS98]{bins}
Martin Raab and Angelika Steger.
\newblock "balls into bins" - {A} simple and tight analysis.
\newblock In {\em Randomization and Approximation Techniques in Computer
  Science, Second International Workshop, RANDOM'98, Barcelona, Spain, October
  8-10, 1998, Proceedings}, pages 159--170, 1998.

\bibitem[Sto77]{S77}
C.~Stone.
\newblock Consistent nonparametric regression.
\newblock {\em Annals of Statistics}, 5:595--645, 1977.

\bibitem[VC71]{vapnik1971uniform}
Vladimir~N Vapnik and A~Ya Chervonenkis.
\newblock On the uniform convergence of relative frequencies of events to their
  probabilities.
\newblock {\em Theory of Probability \& Its Applications}, 16(2):264--280,
  1971.

\end{thebibliography}
\bibliographystyle{alpha}

\newpage
\appendix

\section{Analysis and proofs}\label{sec:gen3}
\label{sec:proof-outline}

The first step in establishing advantage-dependent rates of convergence is to bound the accuracy of empirical estimates of probability mass and bias. This is achieved by a careful choice of large deviation bounds.

\subsection{Large deviation bounds}

Suppose we draw $n$ points $(x_1, y_1), \ldots, (x_n, y_n)$ from $P$. If $n$ is reasonably large, we would expect the empirical mass $\mu_n(S)$ of any set $S \subset \X$, as defined in (\ref{eq:empirical-mass}), to be close to its probability mass under $\mu$. The following lemma, from \cite{ChaudhuriDasgupta2010}, quantifies one particular aspect of this.
\begin{lemma}[\cite{ChaudhuriDasgupta2010}, Lemma 7]
There is a universal constant $c_o$ such that the following holds. Let~$\B$ be any class of measurable subsets of $\X$ of VC dimension $d_0$. Pick any $0 < \delta < 1$. Then with probability at least $1-\delta^2/2$ over the choice of $(x_1, y_1), \ldots, (x_n, y_n)$, for all $B \in \B$ and for any integer $k$, we have
$$ \mu(B) \geq \frac{k}{n} + \frac{c_o}{n} \max \left( k, d_0 \log \frac{n}{\delta} \right)
\ \ \implies \ \ 
\mu_{n}(B) \geq \frac{k}{n} .$$
\label{lemma:points-in-balls}
\end{lemma}

\shay{Why do we use $1-\delta^2/2$ and not $1-\delta$? 
both choices are equivalent (by changing the constant $c_o$)
and $1-\delta$ reads better\ldots }

Likewise, we would expect the empirical bias $\eta_n(S)$ of a set $S \subset \X$, as defined in (\ref{eq:empirical-bias}), to be close to its true bias $\eta(S)$. The latter is defined whenever $\mu(S) > 0$.
\begin{lemma}
There is a universal constant $c_1$ for which the following holds. Let $\mathcal{C}$ be a class of subsets of $\X$ with VC dimension $d_0$. Pick any $0 < \delta < 1$. Then with probability at least $1-\delta^2/2$ over the choice of $(x_1, y_1), \ldots, (x_n, y_n)$, for all $C \in \mathcal{C}$,
  $$ \left| \eta_n(C) - \eta(C) \right| \ \leq \ \Delta(n, \#_n(C), \delta) $$
where $\#_n(C) = |\{i: x_i \in B\}|$ is the number of points in $C$ and 
\begin{equation}
\Delta(n,k,\delta) = c_1 \sqrt{\frac{d_0 \log n + \log (1/\delta)}{k}} .
\label{eq:delta-defn}
\end{equation}
\label{lemma:bias}
\end{lemma}

\Cref{lemma:bias} is a special case\footnote{{Indeed, \Cref{lemma:bias} follows from \Cref{thm:UCECM} 
by plugging in it $\cA = \{\X\times\{+1\}\}, \B = \{C\times\{\pm 1\} : C\in\mathcal{C}\}$.}} of a uniform convergence bound for conditional probabilities (\Cref{thm:UCECM}) 
that we present and prove in \Cref{sec:ucecm}.


\subsection{Proof of Theorem~\ref{thm:pointwise-rate}}

\begin{theorem*}[\Cref{thm:pointwise-rate} restatement] 
There is an absolute constant $C > 0$ for which the following holds.
Let $0 < \delta < 1$ denote the confidence parameter in the $\algname$ algorithm (\Cref{fig:alg}),
and suppose the algorithm is used to define a classifier $g_n$ based on $n$ training points chosen i.i.d.\ from $P$. 
Then, for every point $x\in\supp(\mu)$, if
\[n \geq \frac{C}{\adv(x)} \max \left( \log \frac{1}{\adv(x)}, \ \log \frac{1}{\delta} \right)\]
then with probability at least $1-\delta$ we have that $g_n(x)=g^*(x)$.
\end{theorem*}

\begin{proof}
Define $c_2 = \max(c_1, 1/2) \sqrt{1+c_o}$, where $c_o$ and $c_1$ are the constants from Lemmas~\ref{lemma:points-in-balls} and \ref{lemma:bias}, and take $c_3 = 16 c_2^2$.

Suppose $\eta(x) > 0$; the negative case is symmetric. The set $\B$ of all balls centered at $x$ is easily seen to have VC dimension $d_0 = 1$. By Lemmas~\ref{lemma:points-in-balls} and \ref{lemma:bias}, we have that with probability at least $1-\delta^2$, the following two properties hold for all $B \in \B$:
\begin{enumerate}
\item For any integer $k$, we have $\#_n(B) \geq k$ whenever $n \mu(B) \geq k + c_o \max(k, \log (n/\delta))$.
\item $|\eta_n(B) - \eta(B)| \leq \Delta(n, \#_n(B), \delta)$.
\end{enumerate}
Assume henceforth that these hold.

By the definition of advantage, point $x$ is $(p,\gamma)$-salient for some $p,\gamma> 0$ with $\adv(x) = p\gamma^2$. The lower bound on $n$ in the theorem statement implies that
\begin{equation}
\gamma \geq 2c_2 \sqrt{\frac{\log n + \log (1/\delta)}{np}} ,
\label{eq:gamma}
\end{equation}
or equivalently that $n \cdot \adv(x) \geq 4c_2^2 (\log n + \log (1/\delta))$.

Set $k = np/(1 + w)$. By (\ref{eq:gamma}) we have $np \geq 4 c_2^2 \log (n/\delta)$ and thus $k \geq \log (n/\delta)$. As a result, $np \geq k + w \max(k, \log (n/\delta))$, and by property 1, the ball $B = B(x, r_p(x))$ has $\#_n(B) \geq k$. This means, in turn, that by property 2,
\begin{align*}
\eta_n(B) &\geq \ \eta(B) - \Delta(n, k, \delta)
= \gamma - c_1 \sqrt{\frac{\log (n/\delta)}{k}} \\
&\geq 2c_2 \sqrt{\frac{\log (n/\delta)}{np}} - c_1 \sqrt{\frac{\log (n/\delta)}{k}}
\geq 2c_1 \sqrt{\frac{\log (n/\delta)}{k}} - c_1 \sqrt{\frac{\log (n/\delta)}{k}} \\
&= c_1 \sqrt{\frac{\log (n/\delta)}{k}} \geq \Delta(n, \#_n(B), \delta) .
\end{align*}
Thus ball $B$ would trigger a prediction of $+1$.

At the same time, for any ball $B' = B(x, r)$ with $r < r_p(x)$,
$$ \eta_n(B') \geq \eta(B') - \Delta(n, \#_n(B'), \delta) > -\Delta(n, \#_n(B'), \delta) $$
and thus no such ball will trigger a prediction of $-1$. Therefore, the prediction at $x$ must be $+1$.
\end{proof}

\subsection{Proof of Theorem~\ref{thm:uniform-rate}}

This proof follows much the same outline as that of Theorem~\ref{thm:pointwise-rate}. A crucial difference is that uniform large deviation bounds (Lemmas~\ref{lemma:points-in-balls} and \ref{lemma:bias}) are applied to the class of all balls in $\X$, which is assumed\footnote{This is motivated by finite-dimensional Euclidean space $\mathbb{R}^{D}$, where it holds with $d_0 = D+1$ (\cite{dudley79}).} to have finite VC dimension $d_0$. In contrast, the proof of Theorem~\ref{thm:pointwise-rate} only applies these bounds to the class of balls centered at a specific point, which has VC dimension at most 1 in any metric space.


\subsection{Proof of Theorem~\ref{thm:universal-consistency}}

Recall from (\ref{eq:advantage-set}) that $\X_a$ denotes the set of points with advantage $> a$.
\begin{lemma}
Pick any $0 < \delta < 1$ as a confidence parameter for the $\algname$ estimator of Figure~\ref{fig:alg}. Fix any $a > 0$. If the number of training points $n$ satisfies
$$ n \geq \frac{c_3}{a} \max\left(\log \frac{c_3}{a}, \ \log \frac{1}{\delta} \right), $$
then with probability at least $1-\delta$, the resulting classifier $g_n$ has risk
$$ R(g_n) - R^* \leq \delta + \mu(\X_0 \setminus \X_a) .$$
\label{lemma:advantage-set-convergence}
\end{lemma}
\begin{proof}
From Theorem~\ref{thm:pointwise-rate} \akshay{TODO This lemma should be restated using the constant $C$ from Thm. \ref{thm:pointwise-rate} instead of $c_3$}, we have that for any $x \in \X_a$, 
$$ \pr_n(g_n(x) \neq g^*(x)) \leq \delta^2 ,$$
where $\pr_n$ denotes probability over the choice of training points. Thus, for $X \sim \mu$,
$$ \E_n \E_X 1(g_n(X) \neq g^*(X) | X \in \X_a) \leq \delta^2 ,$$
and by Markov's inequality,
$$ \pr_n [\pr_X (g_n(X) \neq g^*(X) |  X \in \X_a) \geq \delta] \leq \delta.$$
Thus, with probability at least $1-\delta$ over the training set,
$$\pr_X (g_n(X) \neq g^*(X) |  X \in \X_a) \leq \delta .$$
On points with $\eta(x) = 0$, both $g_n$ and the Bayes-optimal $g^*$ incur the same risk. Thus
\begin{align*}
R(g_n) - R^*
&\leq \pr_X(g_n(X) \neq g^*(X) | X \in \X_a) + \pr_X(X \not\in \X_a, \eta(X) \neq 0) \\ 
&\leq \delta + \pr_X(X \in \X_0 \setminus \X_a) + \pr_X(\adv(X) = 0, \eta(X) \neq 0) \\ 
&\leq \delta + \mu(\X_0 \setminus \X_a),
\end{align*}
where we invoke Lemma~\ref{lemma:positive-advantage} for the last step.
\end{proof}

We now complete the proof of Theorem~\ref{thm:universal-consistency}. Given the sequence of confidence parameters $(\delta_n)$, define a sequence of advantage values $(a_n)$ by
$$ a_n = \frac{c_3}{n} \max \left( 2 \log n, \ \log \frac{1}{\delta_n} \right) .$$
The conditions on $(\delta_n)$ imply $a_n \rightarrow 0$.

Pick any $\epsilon > 0$. By the conditions on $(\delta_n)$, we can pick $N$ so that $\sum_{n \geq N} \delta_n \leq \epsilon$. Let $\omega$ denote a realization of an infinite training sequence $(X_1, Y_1), (X_2, Y_2), \ldots$ from $P$. By Lemma~\ref{lemma:advantage-set-convergence}, for any positive integer $N$,
$$ \pr \left(\omega: \exists n \geq N \mbox{\ s.t.\ } R(g_n(\omega)) - R^* > \delta_n + \mu(\X_0 \setminus \X_{a_n}) \right) \leq \sum_{n \geq N} \delta_n \leq \epsilon.$$
Thus, with probability at least $1-\epsilon$ over the training sequence $\omega$, we have that for all $n \geq N$,
$$ R(g_n(\omega)) - R^* \leq \delta_n + \mu(\X_0 \setminus \X_{a_n}) ,$$
whereupon $R(g_n(\omega)) \rightarrow R^*$ (since $\delta_n, a_n \rightarrow 0$ and $\lim_{a \downarrow 0} \mu(\X_0 \setminus \X_a) = 0$). Since this holds for any $\epsilon  >0$, the theorem follows.

\section{Uniform Convergence of Empirical Conditional Measures}
\label{sec:ucecm}

\subsection{Formal Statement}

Let $P$ be a distribution over $X$, and let $\cA,\cB$ be two collections of events.
Consider $n$ independent samples from~$P$, denoted by~$x_1,\ldots,x_n$.
We would like to estimate $P(A \vert B)$ simultaneously for all~$A\in\cA, B\in \cB$.
It is natural to consider the empirical estimates:
\[P_n(A\vert B)=\frac{\sum_i 1_{[x_i\in A \cap B]}}{\sum_i 1_{[x_i\in B]}}.\]
We study when (and to what extent) these estimates provide a good approximation.
Note that the case where $\cB=\{X\}$ (i.e., in which one estimates $P(A)$
using $P_n(A)$ simultaneously for all $A\in\cA$) is handled by the classical VC theory.
Throughout this section we assume that both $\cA,\cB$ have a finite VC-dimension, 
and we let $d_0$ denote an upper bound on both $\mathsf{VC}(\cA)$ and $\mathsf{VC}(\cB)$.

To demonstrate the kinds of statements we would like to derive,
consider the case where each of~$\cA,\cB$ contains only one event:
$\cA=\{A\}$, and $\cB=\{B\}$,
and set $\#_n(B)=\sum_i 1_{[x_i\in B]}$.
A Chernoff bound implies that conditioned on the event that $\#_n(B)>0$, 
the following holds with probability at least~$1-\delta$:
\begin{equation}\label{eq:chernoff}
\left\lvert P(A\vert B) - P_n(A \vert B) \right\rvert \leq \sqrt{\frac{2\log(1/\delta)}{\#_n(B)}}.
\end{equation}
To derive it, use that conditioned on $x_i\in B$, the event $x_i\in A$ has probability $P(A\vert B)$, 
and therefore the random variable ``$\#_n(B)\cdot p_n(A \vert B)$'' has a binomial distribution with parameters $\#_n(B)$ and $P(A\vert B)$.

Note that the bound on the error in \Cref{eq:chernoff} depends on $\#_n(B)$ and therefore is data-dependent.
We stress that this is the type of statement we want:
the more samples belong to $B$, the more certain we are with the empirical estimate.
Thus, we would want to prove a statement as follows:

With probability at least~$1-\delta$,
\[\left(\forall A\in\cA\right)\left(\forall B\in\B\right):\left\lvert P(A \vert B) - P_n(A \vert B) \right\rvert \leq O\left(\sqrt{\frac{d_0 \log(1/\delta)}{\#_n(B)}}\right),\]
where $\#_n(B) = \sum_{i=1}^n 1[x_i\in B]$.

The above statement is, unfortunately, false. 
As an example, consider the probability space defined by drawing $x \sim[n]$ uniformly,
and then coloring $x$ by $c_x\in\{\pm 1\}$ uniformly.
For each $i$ let $B_i$ denote the event that~$i$ was drawn,
and let $A$ denote the event that the drawn color was  $+1$.
(formally, $B_i = \{i\}\times\{\pm 1\}$, and $A=[n]\times\{+1\}$).
One can verify that the VC dimension of $\B=\{B_i : i\leq n\}$ and of $\cA=\{A\}$ is at most $1$.
The above statement fails in this setting:
indeed, one can verify that if we draw $n$ samples from this space 
then with a constant probability there will be some  $j$
such that: 
\begin{itemize}
\item[(i)] $j$ always gets the same color (say $+1$), and 
\item[(ii)] $j$ is sampled at least $\Omega(\log n/\log\log n)$ times\footnote{{This follows from analyzing the maximal bin
in a uniform assignment of $\Theta(n)$ balls into $n$ bins~\cite{bins}}}.
\end{itemize}
Therefore, with constant probability we get that 
\[P_n(A\vert B_i) = 1, P(A\vert B_i)=1/2,\]
and so the difference between the error is clearly $1-(1/2)=1/2$,
which is clearly not upper bounded by $O(\sqrt{\log\log n/\log n})$.

We prove the following (slightly weaker) variant:
\begin{theorem}[UCECM]\label{thm:UCECM}
Let $P$ be a probability distribution over $X$, and let $\cA,\cB$
be two families of measurable subsets of $X$ such that $\mathsf{VC}(\cA),\mathsf{VC}(\cB)\leq d_0$.
Let $n\in\mathbb{N}$, and let $x_1\ldots x_n$ be $n$ i.i.d samples from $P$.
The, the following event occurs with probability at least $1-\delta$:
\[\left(\forall A\in\cA\right)\left(\forall B\in\B\right):\left\lvert P(A \vert B) - P_n(A \vert B) \right\rvert \leq 
\sqrt{\frac{k_o}{\#_n(B)}},\]
where $k_o = 1000 \left(d_0 \log(8n) + \log(4/\delta)\right)$, and\footnote{Note that the above inequality makes sense also when $k(B)=0$,
by identifying $\frac{\cdot}{0}$ as $\infty$, and using the convention that $\infty-\infty=\infty$ and that $\infty\leq \infty$.} $\#_n(B) = \sum_{i=1}^n 1[x_i\in B]$.
\end{theorem}

\paragraph{Discussion.}
\Cref{thm:UCECM} can be combined with \Cref{lemma:points-in-balls}
to yield a bound on the minimal $n$ for which $P_n(A \vert B)$ 
is a non-trivial approximation of $P(A \vert B)$.
Indeed, \Cref{lemma:points-in-balls} implies that 
if $n$ is large enough so that $P(B)=\Omega\left(\frac{d_0\log n}{n}\right)$, 
then the empirical estimate $P_n(A\vert B)$ is a decent approximation.
In the context of the adaptive nearest neighbor classifier, this means that the empirical
biases provide meaningful estimates of the true biases for balls whose measure is $\tilde\Omega\left(\frac{d_0}{n}\right)$.
This resembles the learning rate in realizable settings.
%
%

We remark that a weaker statement than \Cref{thm:UCECM}
can be derived as a corollary of the classical uniform convergence
result~\cite{vapnik1971uniform}. 
Indeed, since the VC dimension of $\{B\cap A : i\in \I\}$ is at most $d_0$, it follows that 
\[P_n(A\vert B)\approx\frac{P(A\cap B) \pm \sqrt{d_0 / n}}{P(B)\pm \sqrt{d_0 / n}}.\]
However, this bound guarantees non-trivial estimates only once $P(B)$ is roughly $\sqrt{d_0  / n}$.
This is similar to the learning rate in agnostic (i.e., non-realizable) settings.

Another major advantage of the uniform convergence bound in \Cref{thm:UCECM} is that it is data-dependent: 
if many points from the sample belong to $B\in \cB$ (i.e.\ $\#_n(B)$ is large), 
then we get better guarantees on the approximation of $P(A\vert B)$ by $P_n(A\vert B)$ for all $A\in\cA$.

\subsection{Proof of \Cref{thm:UCECM}}

As noted above,
the standard uniform convergence bound for VC classes
can not yield the bound in \Cref{thm:UCECM}.
Instead, we use a variant of it due to~\cite{BBL05} which concerns {\it relative deviations}
(see~\cite{BBL05}: Theorem 5.1 and the discussion before Corollary 5.2).
In order to state the theorem, we need the following notation:
Let $\cC$ be a family of subsets of $\X$. We denote by $\mathbb{S}_\cC:\mathbb{N}\to\mathbb{N}$ the {\it growth function} of $\cC$, which is defined by:
\[
\mathbb{S}_\cC(n) = \max\{\lvert \cC|_R\rvert : R\subseteq X, \lvert R\rvert=n\},
\]
where $\cC|_R=\{C\cap R : C\in\cC\}$ is the projection of $\cC$ to $R$.
\begin{theorem}[\cite{BBL05}]\label{thm:ucrel}
Let $\cC$ be a family of subsets of $\X$  and let $P$
be a distribution over $\X$. Then, the following holds with probability $1-\delta$:
\[
(\forall C\in \cC): \lvert P(C)- P_n(C) \rvert \leq  2\sqrt{P_n(C)\frac{\log\mathbb{S}_\cC(2n) + \log(4/\delta)}{n}} + 4\frac{\log\mathbb{S}_\cC(2n) + \log(4/\delta)}{n}. 
\]
\end{theorem}

Set $\cC = \cB\cup \{A\cap B : A\in\cA, B\in\B\}$. 
We prove \Cref{thm:UCECM} by applying \Cref{thm:ucrel} on $\cC$;
to this end we first upper bound $\mathbb{S}_\cC(n)$.
Let $\mathcal{D}= \{A\cap B :A\in\cA, B\in\cB\}$, so that $\cC = \cB \cup \mathcal{D}$. Then:
\begin{align*}
\mathbb{S}_\cC(n) 
\leq \mathbb{S}_\cB(n) + \mathbb{S}_{\mathcal{D}}(n) \leq \mathbb{S}_\cB(n)  +  \mathbb{S}_\cA(n)\mathbb{S}_\cB(n)\leq 2\mathbb{S}_\cA(n)\mathbb{S}_\cB(n)\leq 2{n \choose \leq d_0}^2\leq 2 (2n)^{2d_0},
\end{align*}
where the second inequality follows since $\mathbb{S}_{\mathcal{D}}(n) \leq \mathbb{S}_\cA(n)\mathbb{S}_\cB(n)$,
the second to last inequality follows from the Sauer-Shelah-Perles Lemma, and the last inequality
follows since~${a \choose \leq b} \leq (2a)^b$.
Therefore, applying \Cref{thm:ucrel} on $\cC$ yields that with probability $1-\delta$ the following event holds:
\begin{equation}\label{eq:BBL}
(\forall C\in \cC): \lvert P(C)- P_n(C) \rvert \leq  4\sqrt{P_n(C)\frac{d_0\log 8n + \log(4/\delta)}{n}} + 8\frac{d_0\log 8n + \log(4/\delta)}{n}. 
\end{equation}
For the remainder of the proof we assume that the event in \Cref{eq:BBL} holds and argue that it implies
the conclusion in \Cref{thm:UCECM}.
Let $A\in\cA, B\in\cB$,  
and let $k=n\cdot P_n(B)=\#_n(B)$ denote the number of data points in $B$. 
We want to show that
\begin{equation}\label{eq:qed}
\left\lvert P(A \vert B) - P_n(A \vert B) \right\rvert \leq 
\sqrt{\frac{k_o}{k}},
\end{equation}
where $k_o=1000 \left(d_0 \log(8n) + \log(4/\delta)\right)$.
Let $j=k\cdot P_n(A\vert B) = \#_n(A\cap B)$ denote the number of data points in~$A\cap B$.
We establish \Cref{eq:qed} by showing that 
\[P(A \vert B) \leq  P_n(A\vert B)  + \sqrt{\frac{k_o}{k}}
~~~\text{ and }~~~ P(A \vert B) \geq  P_n(A\vert B)  - \sqrt{\frac{k_o}{k}}.
\]
In the following calculation it will be convenient to denote $D:=d_0 \log(8n) + \log(4/\delta)$. 
By \Cref{eq:BBL} we get:
\begin{align*}
P(A \vert B) &= \frac{P(A\cap B)}{P(B)}\\
			 &\leq \frac{P_n(A\cap B) + 4\sqrt{P_n(A\cap B)\frac{D}{n}} + 8\frac{D}{n}}{P_n(B) - 4\sqrt{P_n(B)\frac{D}{n}} - 8\frac{D}{n}}\\
			 &=\frac{\frac{P_n(A\cap B)}{P_n(B)} + 4\sqrt{\frac{P_n(A\cap B)}{P_n(B)}\frac{D}{nP_n(B)}} + 8\frac{D}{nP_n(B)}}{1 - 4\sqrt{\frac{D}{nP_n(B)}} - 8\frac{D}{nP_n(B)}}s
			 =P_n(A \vert B)\frac{1+ 4\sqrt{\frac{D}{j}} + 8\frac{D}{j}}{1 - 4\sqrt{\frac{D}{k}} - 8\frac{D}{k}},
\end{align*}
where the first inequality follows from \Cref{eq:BBL} and the following equalities are trivial.
Thus,
\begin{equation}\label{eq:16}
P(A \vert B) \leq\frac{j}{k}\Biggl(\frac{1+ 4\sqrt{\frac{D}{j}} + 8\frac{D}{j}}{1 - 4\sqrt{\frac{D}{k}} - 8\frac{D}{k}}\Biggr).
\end{equation}
Next, note that we may assume that $k\geq k_o=1000D$, as otherwise \Cref{eq:qed} trivially holds. Therefore,
\begin{align*}
\frac{1}{1 - 4\sqrt{\frac{D}{k}} - 8\frac{D}{k}} \leq 
1 + 8\sqrt{\frac{D}{k}} + 16\frac{D}{k}. \tag{$(\forall x<\frac{1}{2}):\frac{1}{1-x} \leq 1+2x$}
\end{align*}
Plugging this in \Cref{eq:16}, and using first that $j\leq k$ and then that $1000D\leq k$, yields:
\begin{align*}
P(A \vert B) &\leq \frac{j}{k}\Bigl(1+ 4\sqrt{\frac{D}{j}} + 8\frac{D}{j}\Bigr)\Bigl(1 + 8\sqrt{\frac{D}{k}} + 16\frac{D}{k}\Bigr) \\
&= \frac{j}{k} + 8 \frac{j}{k} \sqrt{\frac{D}{k}} \left( 1 + 2 \sqrt{\frac{D}{k}} \right) + \Bigl( \frac{4\sqrt{j D} + 8D}{k}\Bigr) \Bigl(1 + 4\sqrt{\frac{D}{k}} \Bigr)^2 \\
&\leq \frac{j}{k} + 8 \sqrt{\frac{D}{k}} \left( 1 + 2 \sqrt{\frac{D}{k}} \right) + \Bigl( 4 \sqrt{\frac{D}{k}} + \frac{8D}{k}\Bigr) \Bigl(1 + 4\sqrt{\frac{D}{k}} \Bigr)^2 \\
&\leq \frac{j}{k} + 30\sqrt{\frac{D}{k}} = P_n(A\vert B)  + \sqrt{\frac{k_o}{k}},
\end{align*}
and so 
\[
P(A \vert B) \leq  P_n(A\vert B)  + \sqrt{\frac{k_o}{k}}.
\]
A symmetric argument yields similarly to \Cref{eq:16} that:
\begin{align*}
P(A \vert B) \geq\frac{j}{k}\Biggl(\frac{1 - 4\sqrt{\frac{D}{j}} - 8\frac{D}{j}}{1 + 4\sqrt{\frac{D}{k}} + 8\frac{D}{k}}\Biggr).
\end{align*}
Then, a similar calculation (using the relation $(\forall x > 0):\frac{1}{1+x} \geq 1-2x$) implies that
\[
P(A \vert B) \geq  P_n(A\vert B)  - \sqrt{\frac{k_o}{k}},
\]
which finishes the proof.
\qed

\section{Experimental Results}
\label{sec:experimentappendix}

\subsection{Datasets}

We test $\algname$ on several datasets. The first was the MNIST dataset of 70000 examples (\cite{MNIST}). \yoav{I tried centering the bounding box, but it made the results worse, so I abandoned it. I am using unaltered Euclidean distance on the public MNIST data.}

We also evaluate $\algname$ on the more challenging notMNIST dataset (\cite{notMNIST}), consisting of extracted glyphs of the letters A-J from publicly available fonts. We use the 18724 labeled examples from this set, preprocessed feature-wise to be in $[-\frac{1}{2}, \frac{1}{2}]$ using $x \mapsto \frac{x}{255} - \frac{1}{2}$.

We further use $\algname$ on a challenging binary classification task of independent and continuing interest, involving gene expression data on a population of single cells from different mouse organs collected by the Tabula Muris consortium (\cite{tabulamuris18}, as processed in \cite{MouseAtlasData}). This constitutes 45291 cells (training examples). Each cell has its data collected using one of two approaches. The task is to classify between them. More details follow.

The data are collected using representative protocols of the two currently dominant approaches to isolate and measure single cells: a ``plate"-based approach using microwells on a chip, and a ``droplet"-based approach manipulating cells within droplets using microfluidic technologies. 
Each approach has its own set of technical biases, about which much remains to be understood. Identifying and characterizing these biases to discriminate between such approaches is currently of great interest. 

Both approaches measure effectively the same cells for our purposes, so there is a large decision boundary in the binary classification problem.

\subsection{A note on efficient implementation}

In this paper, we computed the nearest neighbors of data exactly when running $\algname$, to faithfully demonstrate its behavior. 
In practice, this would be done using approximate nearest-neighbor search to build a $k$-NN graph using a small fixed $k$ (say 10), and then using pairwise distances on this graph to compute neighborhoods as needed by $\algname$. 
We tried this (using the nearest-neighbor method of \cite{DCL11}) on notMNIST without substantive differences in the results, and will release this implementation upon publication.

\subsection{Supplemental Figures}

\begin{figure}[th]
    \begin{subfigure}{0.48\textwidth}
    \centering
        \includegraphics[width=\linewidth]{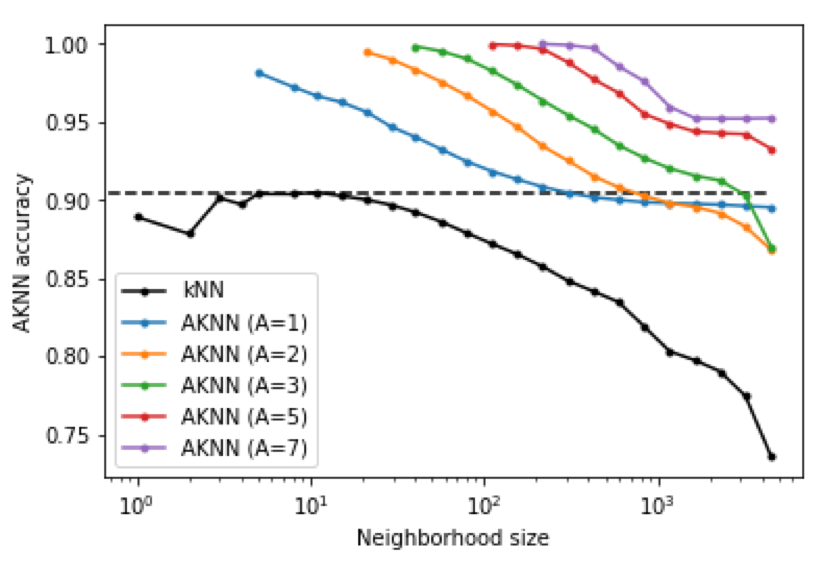}
    \end{subfigure}
    \begin{subfigure}{0.48\textwidth}
    \centering
        \includegraphics[width=\linewidth]{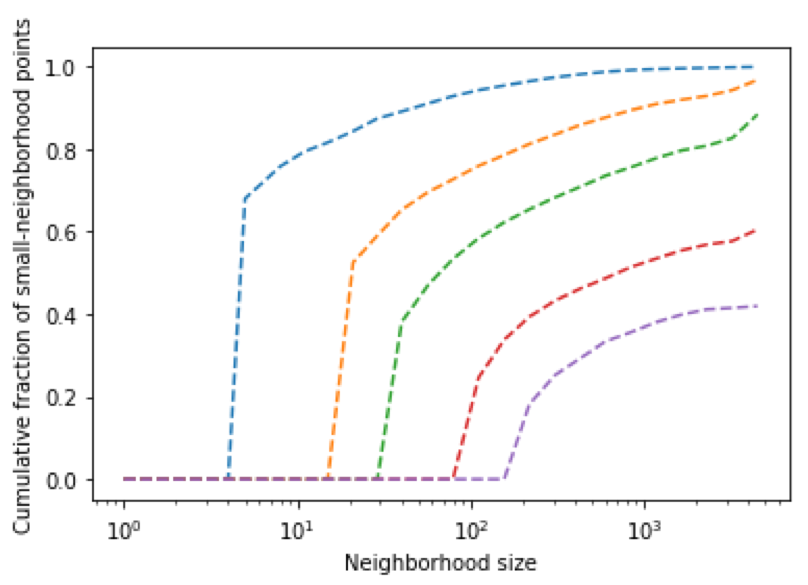}
    \end{subfigure}
    \caption{As Fig. \ref{fig:aknnvsknn}, on single-cell mouse data. $\algname$ is notably accurate on small-neighborhood points at moderate coverage, and performance drops off at higher $k$, with $A$ controlling this frontier. }
    \label{fig:aknnvsknn_muris}
\end{figure}

\begin{figure}
    \begin{subfigure}[t]{0.30\textwidth}
    \centering
        \includegraphics[width=\linewidth]{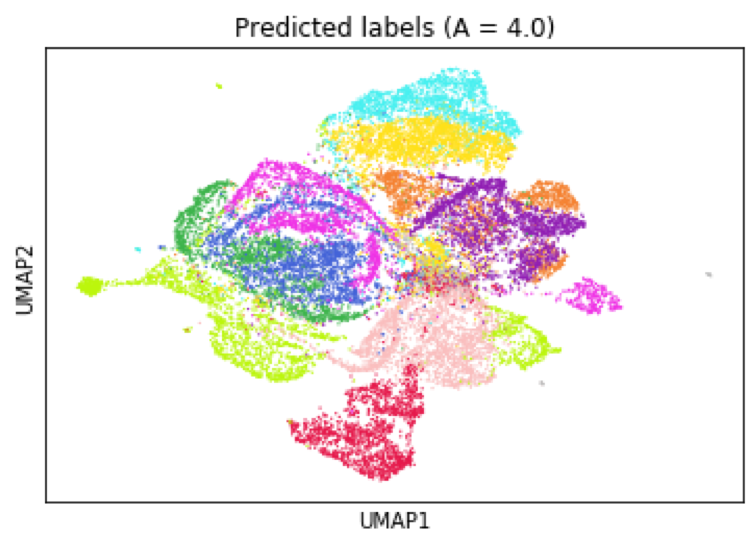} 
    \end{subfigure}
    \begin{subfigure}[t]{0.30\textwidth}
        \centering
        \includegraphics[width=\linewidth]{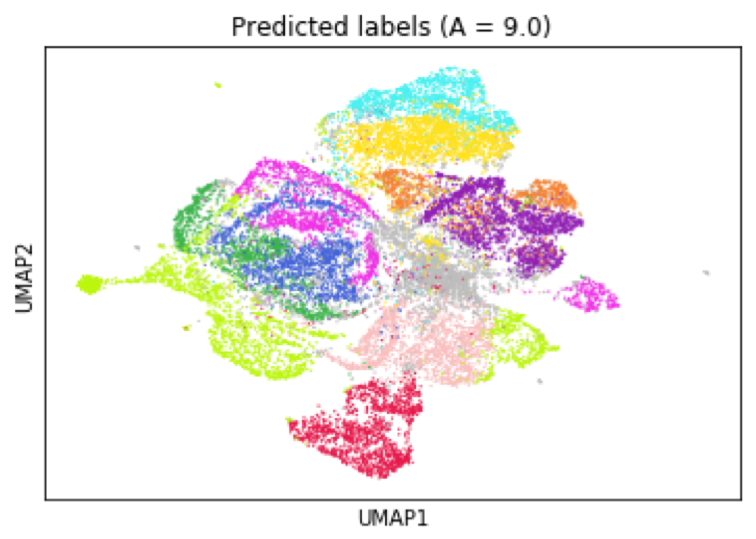} 
    \end{subfigure}
    \begin{subfigure}[t]{0.34\textwidth}
        \centering
        \includegraphics[width=\linewidth]{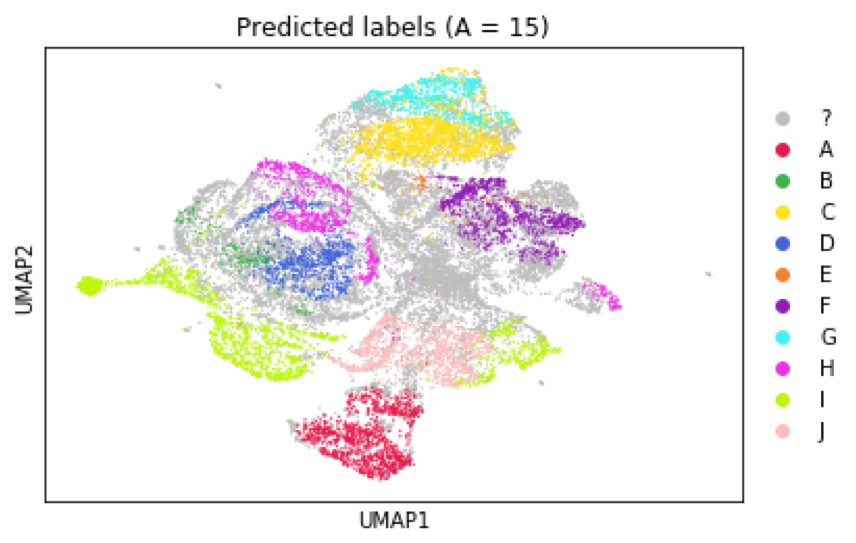} 
    \end{subfigure}
    \hfill
  \caption{$\algname$ predictions on notMNIST, for different settings of $A$.}
  \label{fig:varyingparam}
\end{figure}

\begin{figure}
    \begin{subfigure}[t]{0.32\textwidth}
    \centering
        \includegraphics[width=\linewidth]{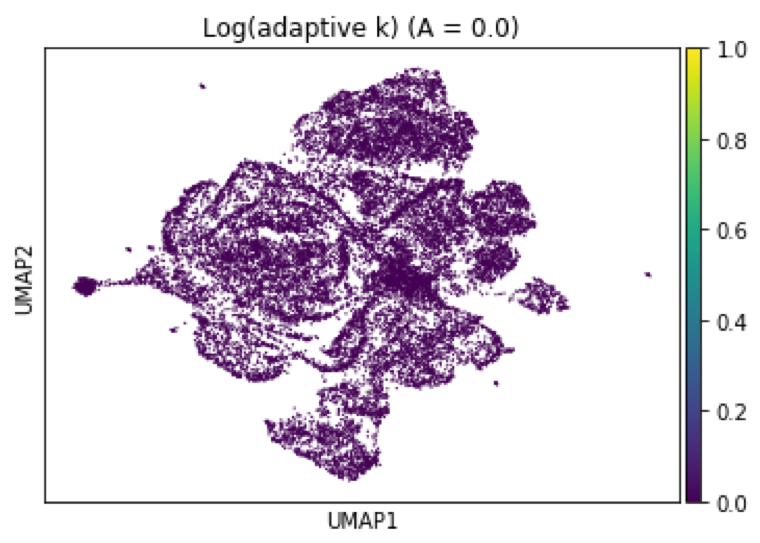} 
    \end{subfigure}
    \begin{subfigure}[t]{0.32\textwidth}
        \centering
        \includegraphics[width=\linewidth]{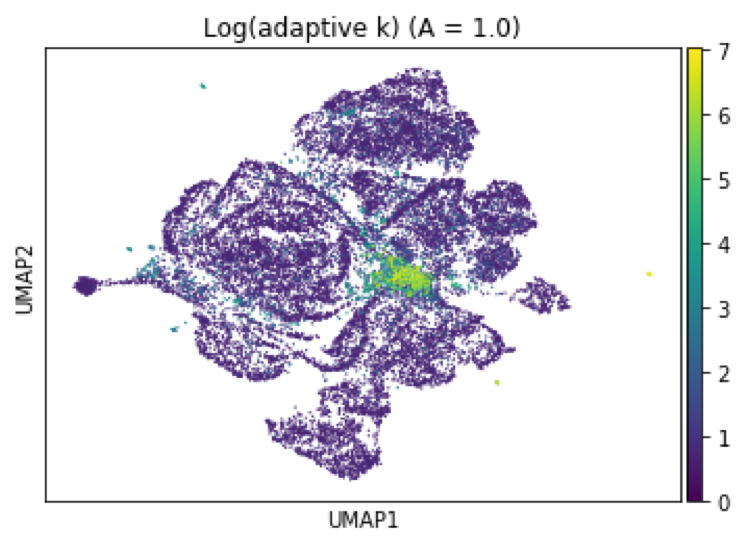} 
    \end{subfigure}
    \begin{subfigure}[t]{0.32\textwidth}
        \centering
        \includegraphics[width=\linewidth]{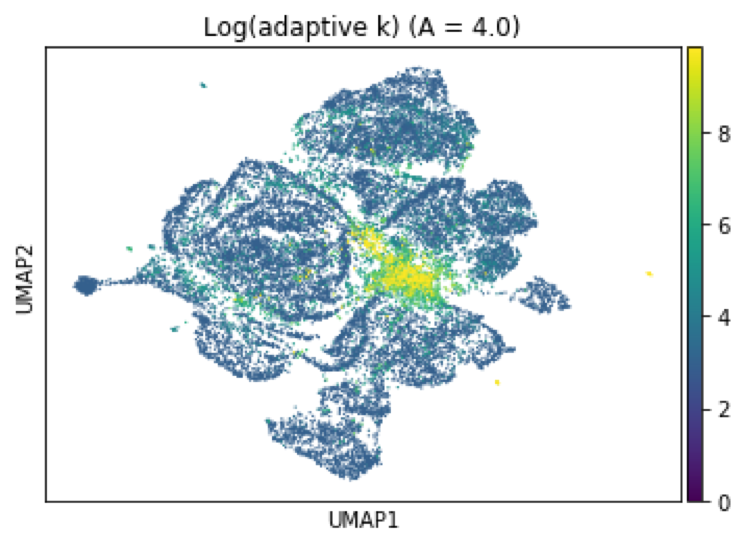} 
    \end{subfigure}
    \hfill
    
    \begin{subfigure}[t]{0.32\textwidth}
    \centering
        \includegraphics[width=\linewidth]{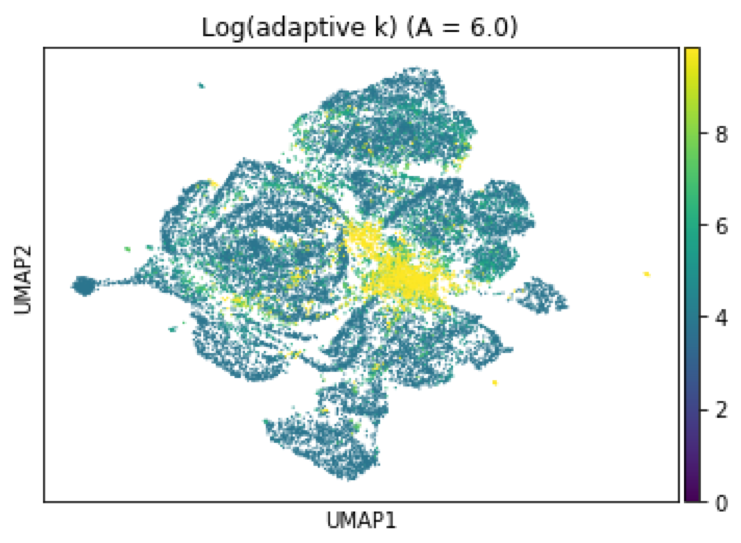} 
    \end{subfigure}
    \begin{subfigure}[t]{0.32\textwidth}
        \centering
        \includegraphics[width=\linewidth]{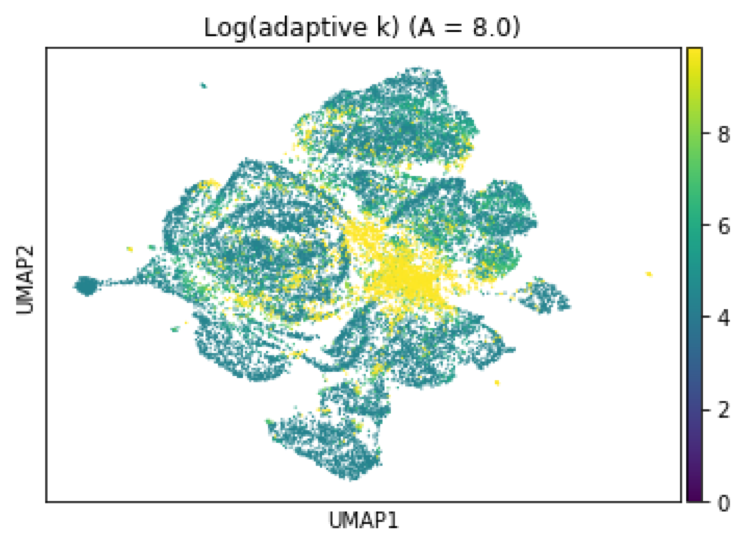} 
    \end{subfigure}
    \begin{subfigure}[t]{0.32\textwidth}
        \centering
        \includegraphics[width=\linewidth]{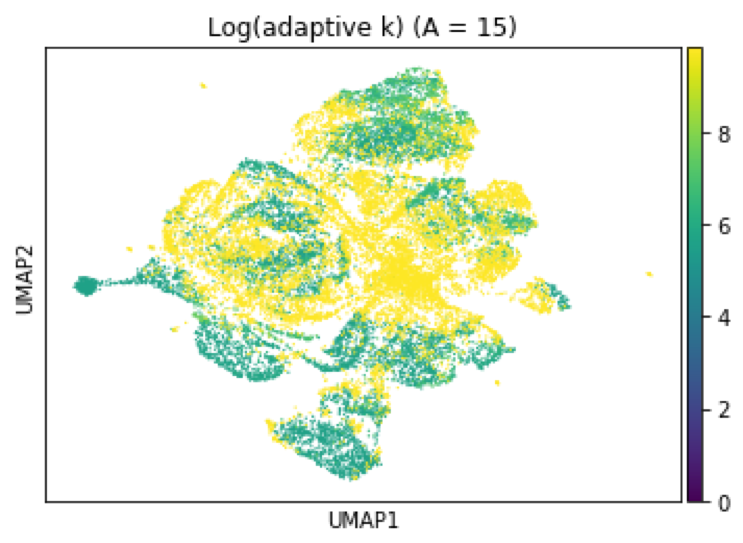} 
    \end{subfigure}
    \hfill
  \caption{$\algname$ neighborhood sizes on notMNIST, in increasing order of $A$, plotted on a log scale. Top left figure ($A = 0$) represents a $1$-NN classifier. Bottom right figure ($A = 15$) shows that many of the points' neighborhoods are maximally large, which can be compared to the right panel of Fig. \ref{fig:varyingparam}.}
  \label{fig:varyingadak_supp}
\end{figure}

    

\begin{figure}
    \begin{subfigure}[t]{0.47\textwidth}
        \centering
        \includegraphics[width=\linewidth]{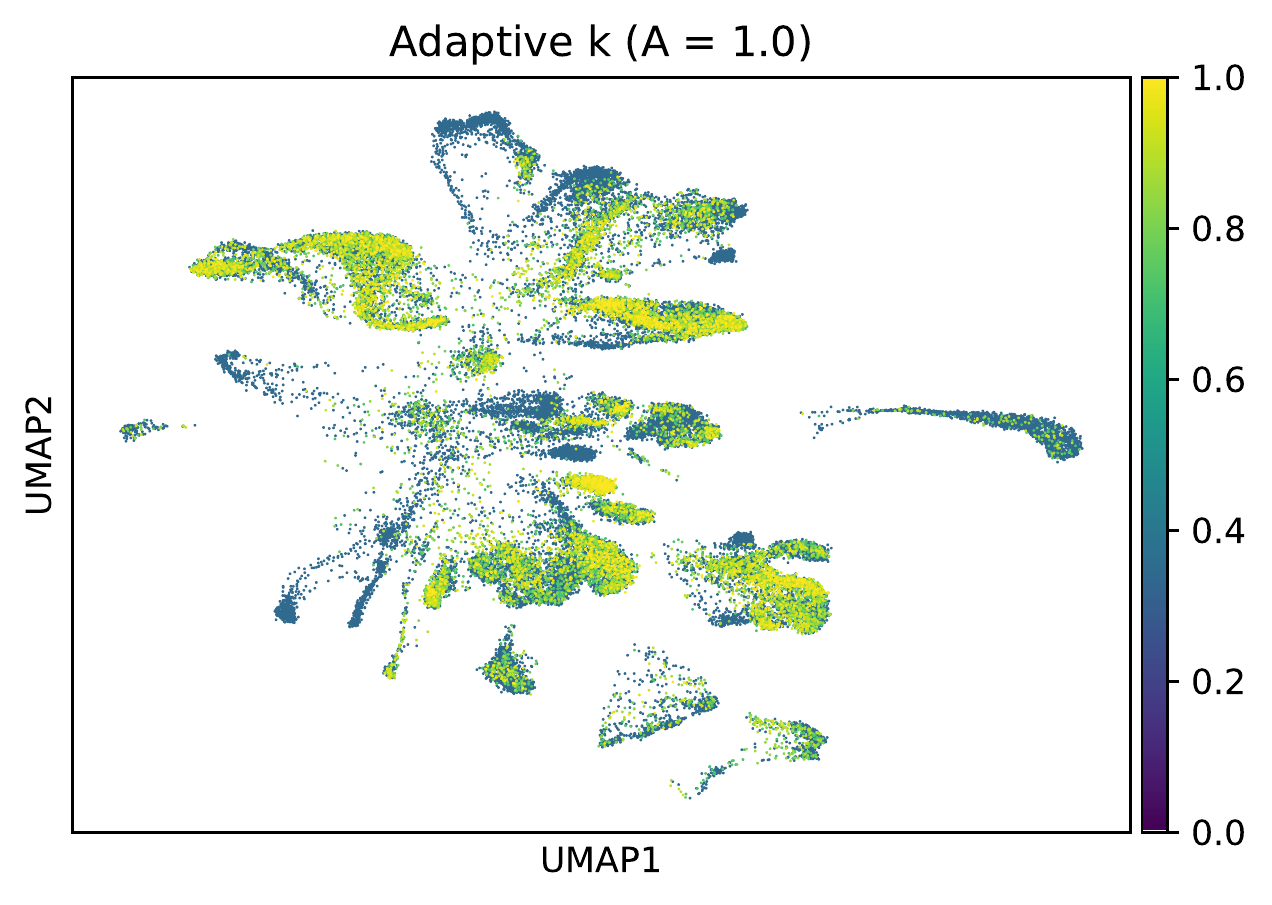} 
    \end{subfigure}
    \begin{subfigure}[t]{0.47\textwidth}
    \centering
        \includegraphics[width=\linewidth]{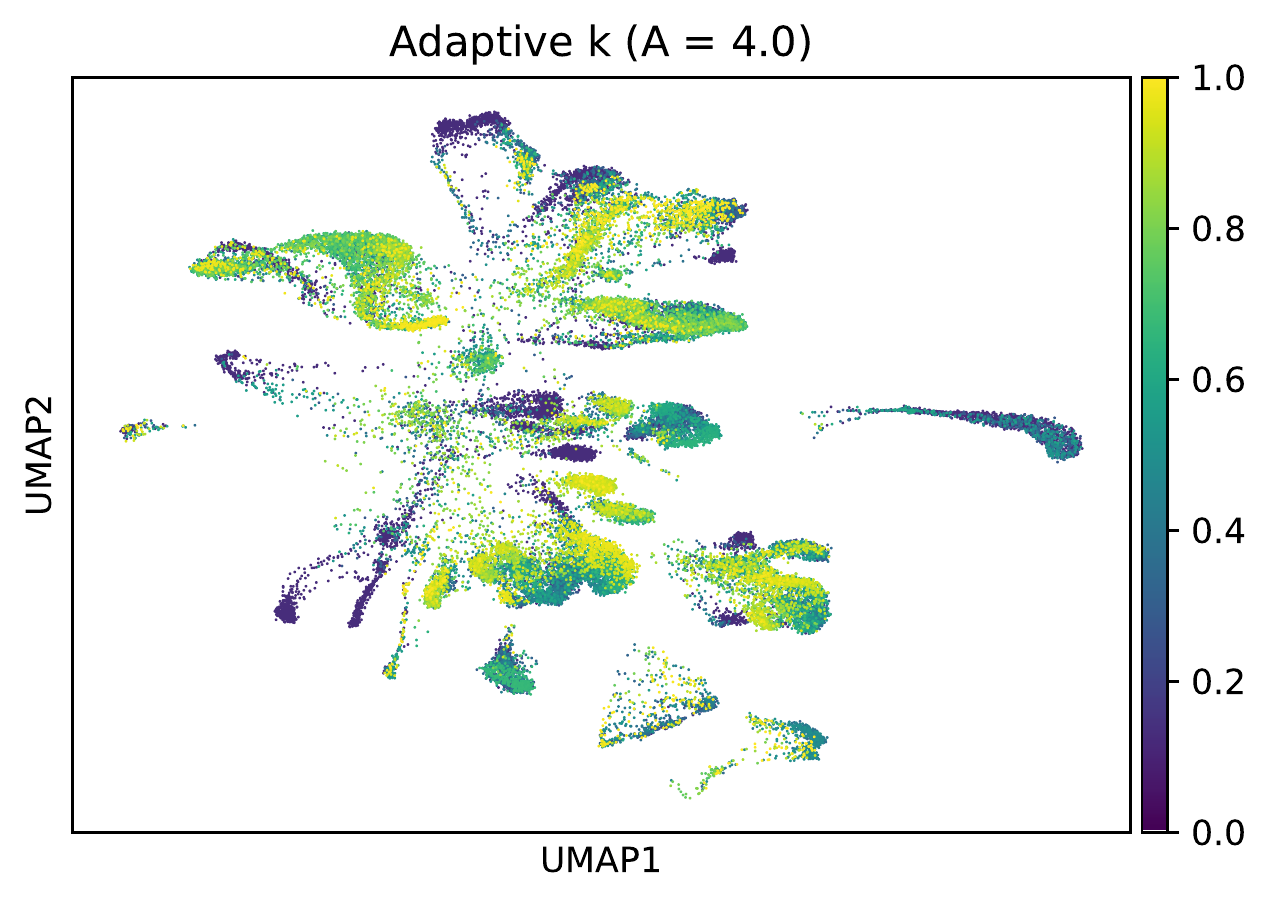} 
    \end{subfigure}
    \hfill
  \caption{As Fig. \ref{fig:varyingadak_supp}, on single-cell mouse data, with the $\algname$ k-values replaced by their quantiles over the data. The relative ordering of the data by $\algname$ neighborhood size is fairly robust to $A$. Zoom in or see \texttt{http://35.239.251.24/aknn/} for details.}
  \label{fig:varyingnormk_muris}
\end{figure}

\end{document}